%% file: main.tex
\documentclass[twoside,11pt]{article}

\usepackage{blindtext}

\usepackage{jmlr2e}
\usepackage{algorithm}
\usepackage{algpseudocode}
\usepackage{tikz}
\usepackage{booktabs}
\usepackage{mathtools}
\usepackage{enumerate}
\usepackage{enumitem}
\usepackage{xr}
\usepackage{mathtools}
\usepackage{amsmath}
\usepackage{bm}
\usepackage{bbm}

\newtheorem{assumption}[theorem]{Assumption}

\newcommand{\argmin}{\arg\min}

\usepackage{lastpage}
\jmlrheading{}{2026}{1-\pageref{LastPage}}{; Revised}{}{}{}

\ShortHeadings{Sparsity-Constraint Optimization via Splicing Iteration}{Sparsity-Constraint Optimization via Splicing Iteration}
\firstpageno{1}

\begin{document}

\title{Sparsity-Constraint Optimization via Splicing Iteration}

\author{\name Jin Zhu, Junxian Zhu \email j.zhu.7@bham.ac.uk \\
       \addr School of Mathematics, University of Birmingham, Birmingham, UK
       \AND
       \name Zezhi Wang, Borui Tang \email wangxq20@ustc.edu.cn \\
       \addr School of Management, University of Science and Technology of China, Hefei, Anhui, China
       \AND
       \name Xueqin Wang, Hongmei Lin \email wangxq20@ustc.edu.cn, hmlin@suibe.edu.cn \\
       \addr School of Statistics and Information, Shanghai University of International Business and Economics, China
       }

\author{\name Jin Zhu\textsuperscript{1}, Junxian Zhu\textsuperscript{2} \email j.zhu.7@bham.ac.uk, junxian@nus.edu.sg \\
  \name Zezhi Wang\textsuperscript{3}, Borui Tang\textsuperscript{3} \email \{homura, tangborui\}@mail.ustc.edu.cn \\
  \name Xueqin Wang\textsuperscript{3}, Hongmei Lin\textsuperscript{4} \email wangxq20@ustc.edu.cn, hmlin@suibe.edu.cn \\
  \textsuperscript{1} \addr School of Mathematics, University of Birmingham, Birmingham, UK\\
  \textsuperscript{2} \addr Saw Swee Hock School of Public Health, National University of Singapore, Singapore \\
  \textsuperscript{3} \addr Department of Statistics and Finance/International Institute of Finance, School of Management, University of Science and Technology of China, Hefei, Anhui, China \\
  \textsuperscript{4} \addr School of Statistics and Information, Shanghai University of International Business and Economics, Shanghai, China 
}

\editor{My editor}

\maketitle
\begingroup\renewcommand\thefootnote{*}
\footnotetext{Jin Zhu and Junxian Zhu contributed equally. Xueqin Wang and Hongmei Lin are the corresponding author.}
\endgroup
\begin{abstract}
Sparsity-constrained optimization underlies many problems in signal processing, statistics, and machine learning. State-of-the-art hard-thresholding (HT) algorithms rely on an appropriately selected continuous step-size parameter to ensure convergence. In this paper, we propose a naturally convergent iterative algorithm, SCOPE (\underline{S}parsity-\underline{C}onstrained \underline{O}ptimization via S\underline{p}licing It\underline{e}ration). The algorithm is capable of optimizing nonlinear differentiable objective functions that are strongly convex and smooth on low-dimensional subspaces. SCOPE replaces the gradient step with a splicing operation guided directly by the objective value, thereby eliminating the need to tune any continuous hyperparameter. Theoretically, it achieves a linear convergence rate and recovers the true support set when the sparsity level is correctly specified. We also establish parallel theoretical results without relying on restricted-isometry-property-type conditions. We apply SCOPE's versatility and power to solve sparse quadratic optimization, learn sparse classifiers, and recover sparse Markov networks for binary variables. With our \textsf{C++} implementation of SCOPE, numerical experiments on these tasks show that it achieves superior support recovery performance, confirming both its algorithmic efficiency and theoretical guarantees.
\end{abstract}


\input{diy_tex_command.tex}
\section{Introduction}

This paper aims to develop an algorithm to solve the sparsity-constrained optimization:  
\begin{equation}\label{originalProblem}
	\argmin_{\bm{\theta} \in \mathbb{R}^p} f(\bm{\theta}),\textup{ s.t. } \|\bm{\theta}\|_0 \leq s,
\end{equation} 
where $f: \mathbb{R}^p \rightarrow \mathbb{R}$ is a differentiable objective function with strong convexity and smoothness in low dimensional subspaces, $\bm{\theta}$ is a $p$-dimensional parameter vector, $\| \bm{\theta} \|_0$ is the $\ell_0$-norm that counts the non-zero coordinates in $\bm{\theta}$, and $s$ is a given integer controlling the sparsity of the solution of~\eqref{originalProblem}. 
As an interesting topic in optimization \citep{beck2013sparsity}, problem~\eqref{originalProblem} acquires increasing significance in machine learning, statistics, and signal processing nowadays, in which it also refers to the sparsity learning or subset selection. We exemplify with compressive sensing that aims to recover a sparse signal vector $\bm{\theta} \in \mathbb{R}^p$, which is formulated as: 
\begin{equation}\label{eq:quardticobj}
	\argmin_{\bm{\theta} \in \mathbb{R}^p} \|\mathbf{y} - \mathbf{X}\bm{\theta}\|^2_2,\textup{ s.t. } \|\bm{\theta}\|_0 \leq s,
\end{equation}
where $\|\cdot\|_2$ is the $\ell_2$-norm, $\mathbf{y} \in \mathbb{R}^n$ are the observations, $\mathbf{X} \in \mathbb{R}^{n\times p}$ is the sensing matrix. 
Although problem~\eqref{eq:quardticobj} has received extensive studies~\citep*[see, e.g.,][]{donoho2006compressed, tropp2007signal, blumensath2008iterative, needell2009cosamp, foucart2011hard}, 
the study of \eqref{originalProblem} is still limited because of the flexibility of the objective function. Even more difficult, the constraint function is non-convex, and in fact, is non-continuous, rendering problem~\eqref{originalProblem} quite challenging.




\subsection{Literature review}

We roughly categorize the main existing methods for \eqref{originalProblem} into two categories: searching-based methods and hard thresholding (HT) based methods.
\begin{itemize}[leftmargin=*]
    \item \underline{\textit{Searching-based methods}} aim to identify a coordinate set of size $s$ that minimizes $f(\cdot)$. The most direct approach is exhaustive enumeration over all coordinate sets with cardinality $s$, which yields an exact solution but is generally computationally intensive when the dimension is large. To improve computational efficiency, \citet{shalev-shwartz_trading_2010, liu2014forward} proposed greedy searching methods that sequentially appends/remove coordinates until an $s$-sparse solution is obtained. To avoid greedy, one approach is reformulating~\eqref{originalProblem} as a mixed-integer nonlinear programming problem \citep[MINLP,][]{floudas1995nonlinear} and finds its global solution via nonlinear branch-and-bound algorithms \citep{belotti2013mixed}, which are implemented in modern optimization solvers such as \texttt{SCIP} and \texttt{GUROBI} \citep{bestuzheva2023global, gurobi}. However, the presence of integer variables renders MINLP problems NP-hard in general, and global optimization methods may require exploring large branch-and-bound trees in worst-case scenarios \citep{belotti2013mixed}. 
    \item \underline{\textit{HT-based algorithms}} directly exploit the sparsity level $s$ and use hard-thresholding operators to maintain $s$-sparse iterates. As one of the most classical HT-based algorithms, iterative hard thresholding (IHT) was first proposed by \citet{blumensath2008iterative} for quadratic objectives and was later extended to more general settings \citep{beck2013sparsity, jain2014iterative}. IHT is computationally efficient and scalable in practice, which has inspired various variants and extensions \citep{blumensath2010normalized,yuan2020dual,chaudhuri2022iterative,ida2024fast}. Another representative line of HT-type methods apply an additional \emph{debiasing} step on the selected support during iteration. Such a debiasing idea was introduced early in compressive sensing, e.g., in compressed sampling matching pursuit \citep[CoSaMP,][]{needell2009cosamp} and hard thresholding pursuit \citep[HTP,][]{foucart2011hard}, and it has since been adapted to solve \eqref{originalProblem}. For instance, \citet{bahmani_greedy_2011} proposed gradient support pursuit (GraSP) which can be viewed as a extension of CoSaMP to general objective functions, and \citet{yuan2017gradient} proposed gradient hard thresholding pursuit (GraHTP) that generalizes HTP. Empirically, the debiasing step has been found to improve performance in terms of computational efficiency and support recovery \citep{zhao2022lagrange}. Motivated by this benefit, more sophisticated debiasing strategies have been developed to enhance sparsity recovery and/or further accelerate convergence \citep[see, e.g.,][]{huang2017constructive, zhou_global_nodate}.
\end{itemize}

\subsection{Motivation, Proposal and Contribution}

This paper is motivated by the existing HT-based methods discussed previously. Generally, HT-based algorithms involve a gradient descent step of the form $\param^{t+1} \leftarrow \param^t - \tau \nabla \loss(\param^t)$, where $\param^t$ denotes the iterate at iteration $t$, and $\tau$ is a continuous step-size parameter. The effectiveness of HT-based algorithms relies on a properly chosen $\tau$, which depends on the objective function $\loss$ and may be difficult to determine in practice, especially when $p$ is large. If the step size is misspecified, the performance of the algorithm may deteriorate (see Section~\ref{sec:experiments}); moreover, the iterates may even become periodic \citep{foucart2011hard}.

This paper aims to develop a naturally convergent algorithm with a provably accurate solution for problem~\eqref{originalProblem}. To this end, we propose a new iterative algorithm, the \underline{s}parsity-\underline{c}onstrained \underline{o}ptimization via s\underline{p}licing it\underline{e}ration (SCOPE). At each iteration, SCOPE leverages zeroth-order and first-order information at the current solution and performs local swapping to overcome the challenges of general objective functions and the sparsity constraint. Besides, in each iteration, SCOPE exploits the objective values to guide the local swapping to ensure the objective value monotonously decreases with the number of iterations; in conjunction with the fact that support sets with cardinality $s$ are finite, SCOPE must converge without carefully setting continuous tuning parameters to reach convergence. This is a distinctive feature of SCOPE. 

The main contributions of this paper are as follows:
\begin{enumerate}[leftmargin=*]
    \item \textbf{A tuning-free algorithm.} We propose SCOPE, a new 
    iterative algorithm for problem~\eqref{originalProblem}. SCOPE 
    replaces the gradient-descent-plus-thresholding template of 
    existing HT-based methods with an objective-guided splicing 
    operation that admits \emph{best-candidate} acceptance (see 
    Section~\ref{sec:algorithm}). As a consequence, SCOPE is free of 
    the continuous step-size parameter that governs HT-type methods, 
    and its only hyperparameter $k_{\max}$ is discrete and bounded 
    by $s$.
    
    \item \textbf{Support recovery and linear convergence under a 
    weaker RIP-type condition.} Under 
    Assumptions~\ref{con:convex-smooth}--\ref{con:bound-gradient}, 
    we prove that SCOPE exactly recovers $\mathcal{A}^*$ 
    (Theorem~\ref{thm:recovery}) and converges linearly 
    (Theorem~\ref{thm:convergence-rate}). The restricted condition 
    number is required to satisfy $M_{3s}/m_{3s} \leq 1.49$, which is 
    strictly weaker than the corresponding bounds in 
    \citet{needell2009cosamp}, \citet{bahmani_greedy_2011}, and 
    \citet{yuan2017gradient}. Moreover, we derive an explicit 
    iteration complexity bound 
    (Theorem~\ref{thm:convergence}) and show geometric convergence 
    of the iterates to the oracle solution.
    
    \item \textbf{Guarantees without RIP-type conditions.} By 
    relaxing the working sparsity level $s$ beyond the true sparsity 
    $s^*$, we establish support recovery 
    (Theorem~\ref{thm:support_recovery_relax}) and linear 
    convergence (Theorem~\ref{thm:convergence_rate_relax}) without 
    any RIP-type restriction on $M/m$. Compared with 
    \citet{jain2014iterative} and \citet{yuan2017gradient}, the 
    required relaxation scales more slowly with the restricted 
    condition number.
    
    \item \textbf{Applications and empirical validation.} We 
    instantiate SCOPE on three benchmark tasks---compressed sensing, 
    sparse classification, and sparse Ising model recovery---and 
    release an open-source C++ implementation. On these tasks, SCOPE 
    requires fewer samples for exact support recovery than 
    GraSP, GraHTP, and Lasso-based relaxations, while its runtime is 
    competitive with HT-based methods and $10$--$20\times$ faster 
    than Lasso-based baselines.
\end{enumerate}

\subsection{Organization}

The rest of this article is organized as follows. 
Section~\ref{sec:algorithm} fixes notation and presents the SCOPE 
algorithm. Section~\ref{sec:theory} develops the theoretical 
properties of SCOPE, covering both support recovery and convergence 
analysis, with and without RIP-type conditions. 
Section~\ref{sec:application} instantiates SCOPE on three 
canonical sparse learning tasks---compressed sensing, sparse 
classification, and sparse Ising model recovery---and verifies the 
required assumptions under standard conditions. 
Section~\ref{sec:experiments} evaluates SCOPE numerically against 
an exact commercial solver and against state-of-the-art HT-based 
and convex-relaxation baselines. Section~\ref{sec:conclusion} 
concludes. Complete technical proofs are deferred to the appendix.

\section{Algorithm}\label{sec:algorithm}

\subsection{Notation}\label{sec:notation}
We first collect the notation used throughout the paper. Bold 
lowercase letters (e.g., $\mathbf{x}$, $\bm{\theta}$) denote vectors, 
and bold uppercase letters (e.g., $\mathbf{X}$, $\mathbf{Y}$) denote 
matrices. Let $I(\cdot)$ be the indicator function, $[p] \coloneqq 
\{1, 2, \dots, p\}$, and $\mathcal{A}, \mathcal{B} \subseteq [p]$. 
For a vector $\mathbf{x} \in \mathbb{R}^p$, we write $\|\mathbf{x}\|_2$ 
(or simply $\|\mathbf{x}\|$) for its Euclidean norm, $\|\mathbf{x}\|_0$ 
for the number of nonzero coordinates, and $\|\mathbf{x}\|_{\infty}$ 
for the maximum absolute entry. The support of $\mathbf{x}$ is 
$\operatorname{supp}(\mathbf{x}) = \{j \mid \mathbf{x}_j \neq 0\}$, 
and $\operatorname{diag}\{\mathbf{x}\}$ denotes the diagonal matrix 
with $\mathbf{x}$ on the diagonal. For a positive integer $t$, 
$\mathcal{H}_t(\cdot): \mathbb{R}^p \to \mathbb{R}^p$ is the hard 
thresholding operator that retains the $t$ coordinates of largest 
magnitude and zeros out the rest.

For the objective function $f: \mathbb{R}^p \to \mathbb{R}$, we write 
$\nabla f(\mathbf{x})$ and $\nabla^2 f(\mathbf{x})$ for its gradient 
and Hessian at $\mathbf{x}$. Given index sets $\mathcal{A}, \mathcal{B} 
\subseteq [p]$, $\mathbf{x}_{\mathcal{A}} \in \mathbb{R}^{|\mathcal{A}|}$ 
is the subvector of $\mathbf{x}$ indexed by $\mathcal{A}$, while 
$\mathbf{x}|_{\mathcal{A}} \in \mathbb{R}^p$ is the vector whose 
$j$-th coordinate equals $\mathbf{x}_j$ if $j \in \mathcal{A}$ and 
zero otherwise; $\mathbf{X}_{\mathcal{A}}$, $\mathbf{X}_{\mathcal{A},\cdot}$, 
and $\mathbf{X}_{\mathcal{A},\mathcal{B}}$ denote, respectively, the 
submatrix of $\mathbf{X}$ consisting of the columns in $\mathcal{A}$, 
the rows in $\mathcal{A}$, and the rows in $\mathcal{A}$ and columns 
in $\mathcal{B}$. We write $\nabla_{\mathcal{A}} f(\mathbf{x}) 
\coloneqq (\nabla f(\mathbf{x}))_{\mathcal{A}}$ and, analogously, 
$\nabla_{\mathcal{B}} \nabla_{\mathcal{A}} f(\mathbf{x}) \coloneqq 
((\nabla^2 f(\mathbf{x}))_{\mathcal{A},\cdot})_{\mathcal{B}}$; in 
particular, $\nabla^2_{\mathcal{A}} f(\mathbf{x}) \coloneqq 
\nabla_{\mathcal{A}} \nabla_{\mathcal{A}} f(\mathbf{x})$. Finally, 
$\bm{\theta}^*$ denotes the ground-truth $s$-sparse parameter and 
$\mathcal{A}^* \coloneqq \operatorname{supp}(\bm{\theta}^*)$ its 
support.

\subsection{The SCOPE algorithm}\label{sec:scope-algo}

The central idea of SCOPE is to iteratively refine the current active set and construct a new support with improved objective value. Starting from an initial support $\A^0 \subseteq [p]$ with $|\A^0|=s$, the algorithm generates a sequence of support sets with cardinality $s$: $\A^0, \A^1, \ldots$, and returns the final support set in this sequence as the output. Suppose $\A^t$ is the $t$-th solution, a heuristic strategy to improve upon $\A^t$ is replacing some irrelevant elements in $\A^{t}$ with an equal number of top relevant elements in $\I^t \coloneqq (\A^t)^c$. The effectiveness of such a strategy relies on a quantitative notion of ``relevance'' that measures the contribution of each coordinate to the objective function. The strategy, named splicing, was introduced in \citet{abess} for compressive sensing, where a relevance criterion can be derived in closed form. However, deriving a suitable relevance  for general objective functions is more challenging as it must faithfully capture the behavior of the objective function, while maintaining computational efficiency.  

To meet the requirements, we define the relevance as
\begin{equation}\label{eq:sacrifice}
	\xi^t_j \propto 
	\begin{cases}
		(\bm{\theta}^t_j)^2, & j \in \A^t \\ 
		{[\nabla_j f(\bm{\theta}^t)]^2,} & j \in \I^t
	\end{cases},
\end{equation}
where $\bm{\theta}^t \coloneqq \mathop{\argmin}\limits_{\operatorname{\supp}\{\param\}=\A^t} f(\bm{\theta})$ is the current solution under $\A^t$. This definition incorporates both zeroth- and first-order information of $f$. For $j \in \A^t$, a larger magnitude of $\theta^t_j$ indicates greater importance within the current support. For $j \in \I^t$, a larger gradient magnitude of $\nabla_j f(\bm{\theta}^t)$ suggests a stronger potential for objective reduction if the coordinate is activated. Moreover, the relevance can be easily computed as it is a closed-form expression with respect to $\bm{\theta}^t$ and $\nablaf(\bm{\theta}^t)$. Among them, obtaining $\bm{\theta}^t$ requires optimizing objective function in a low dimensional parameter space since only $s$ parameters have non-zero values; as for $\nablaf(\bm{\theta}^t)$, their expressions can be easily obtained. 

Now we depict the splicing procedure with the relevance score. For a fixed $k \in \{1,\ldots,k_{\max}\}$, define two sets\footnote{Notice that, it might be possible that $|\mathcal{S}^{(k)}_{\A}|$ or $|\mathcal{S}^{(k)}_{\I}|$ is not equal to $k$ because $\sacri_j^t$ has the same value. Although this extreme case would not appear in the application in this paper, one possible solution for this case is to randomly choose some of them to ensure $|\mathcal{S}^{(k)}_{\A}| = |\mathcal{S}^{(k)}_{\I}| = k$.}:
\begin{equation}\label{eq:exchange-set}
\begin{aligned}
    \mathcal{S}_{\A}^{(k)} &=\{j \in \A^t: \textup{$\xi^t_j$ is in the smallest $k$ of $\{\xi^t_j\}_{j\in\A^t}$}\},
    \\
    \mathcal{S}_{\I}^{(k)}
    &=\{j \in \I^t: \textup{$\xi^t_j$ is in the largest $k$  of $\{\xi^t_j\}_{j\in\I^t}$}\}.
\end{aligned}
\end{equation}
The candidate support with cardinality $s$ is then constructed as $\widetilde{\A}^{(k)} \coloneqq (\A^t \setminus \mathcal{S}_\A^{(k)}) \cup \mathcal{S}_\I^{(k)}$ and evaluated via $\widetilde{f}^{(k)} \coloneqq \min\limits_{\operatorname{supp}\{\bm{\theta}\} = \widetilde{\A}^{(k)}} f(\bm{\theta})$. Among all $k \in \{1,\ldots,k_{\max}\}$, we select $k^* \coloneqq \mathop{\argmin}\limits_{k\in [s]}\widetilde{f}^{(k)}$. If $\widetilde{f}^{(k^*)} < f(\bm{\theta}^t)$, then $\widetilde{\A}^{(k^*)}$ may refine $\A^t$ and we should accept this candidate support by performing update $\A^{t+1} \leftarrow \A^{k^*}$. If $\widetilde{f}^{(k^*)} \geq f(\bm{\theta}^t)$, no candidate support is better than $\A^t$, then we set $\A^t$ as the algorithmic output.
We summarize the splicing procedure depicted above in Algorithm~\ref{algo:main}.

We make two remarks on Algorithm~\ref{algo:main}. 
\begin{itemize}[leftmargin=*]
    \item \textbf{(Finite-step convergence)} For the active set sequence $\A^0, \A^1, \ldots, $ generated in Algorithm~\ref{algo:main}, it must have a terminal $\A^T$ where $T$ is a finite integer. This follows from the strict decrease of the objective value and the finiteness of $s$-cardinality subsets of $[p]$. Hence, SCOPE converges without additional assumptions. 
    \item \textbf{(Tuning-free property)}  Algorithm~\ref{algo:main} involves two inputs: (i) the maximum splicing size $k_{\max}$ and (ii) an initial support $\A^0$. The parameter $k_{\max}$ is discrete. Setting $k_{\max}=s$ ensures the theoretical guarantees established in Section~\ref{sec:theory}, and this choice is adopted throughout our experiments. Regarding the initialization, our theoretical analysis shows that the choice of $\A^0$ affects only the number of iterations required for termination, but not the convergence rate. In practice, we initialize $\A^0$ using a gradient-based rule. Specifically, define $\xi^0_j = [\nabla_j f(\mathbf{0})]^2, \quad j=1,\ldots,p,$ and let $\A^0$ consist of the indices corresponding to the $s$ largest values of $\{\xi^0_j\}_{j=1}^p$. With this initialization strategy and the natural choice $k_{\max}=s$, SCOPE does not require hyperparameter tuning.
\end{itemize}

\begin{algorithm}[htbp]
\caption{\underline{S}parsity-\underline{C}onstrained \underline{O}ptimization via S\underline{p}licing It\underline{e}ration (SCOPE)}
\label{algo:main}
\begin{algorithmic}[1]
\Require An initial support $\A^0$ with cardinality $s$, and the maximum splicing size $k_{\max}(\leq s)$.
\Ensure $(\bm{\theta}^{t}, \A^{t})$.
\State Initialize: $t \leftarrow 0$, $\I^{0} \leftarrow (\A^{0})^{c}$,
$\bm{\theta}^{0} \leftarrow \argmin\limits_{\operatorname{supp}\{\bm{\theta}\} = \A^0} f(\bm{\theta})$.
\Repeat \Comment{Outer splicing iteration}
    \State $t \leftarrow t+1$, $\A^t \leftarrow \A^{t-1}$, $\param^t \leftarrow \param^{t-1}$, $L \leftarrow f(\bm{\theta}^{t})$.
    \State Update the relevance of each coordinate via Equation~\eqref{eq:sacrifice}.
    \For{$k = 1, \ldots, k_{\max}$} \Comment{Evaluate candidate swaps}
        \State Select the $k$ coordinates 
        $\mathcal{S}_\A^{(k)}, \mathcal{S}_\I^{(k)}$
        to be swapped via Equation~\eqref{eq:exchange-set}.
        \State $\widetilde{\A}^{(k)} \leftarrow (\A^{t} \backslash \mathcal{S}_\A^{(k)}) \cup \mathcal{S}_\I^{(k)}$, $\widetilde{\bm{\theta}}^{(k)} \leftarrow
        \argmin\limits_{\operatorname{supp}\{\bm{\theta}\} = \widetilde{\A}^{(k)}} f(\bm{\theta})$.
        \If{$L > f(\widetilde{\bm{\theta}}^{(k)})$}
            \State $(L, \A^{t}, \bm{\theta}^{t}) \leftarrow (f(\widetilde{\bm{\theta}}^{(k)}), \widetilde{\A}^{(k)}, \widetilde{\bm{\theta}}^{(k)})$.
        \EndIf
    \EndFor
\Until{$\A^{t} = \A^{t-1}$}
\end{algorithmic}
\end{algorithm}
Although SCOPE inherits the idea of coordinate 
exchange from \citet{abess}, the two algorithms differ substantially 
in methodology, scope, and theory. Table~\ref{tab:abess-vs-scope} 
summarizes the key distinctions.

\begin{table}[htbp]
\centering
\small
\caption{Methodological and theoretical comparison between SCOPE 
and the splicing algorithm of \citet{abess}.}
\label{tab:abess-vs-scope}
\begin{tabular}{@{}lll@{}}
\toprule
 & \citet{abess} & SCOPE (this paper) \\
\midrule
Target problem & Compressed sensing (quadratic $f$) & General differentiable $f$ \\[2pt]
Relevance score for $j \in \A^t$ & $f(\param^t) - \min_{\operatorname{supp}(\bm{\alpha})=\{j\}} f(\param^t + \bm{\alpha})$ & $(\param^t_j)^2$ (closed form) \\[2pt]
Relevance score for $j \in \I^t$ & univariate subproblem & $[\nabla_j f(\param^t)]^2$ (closed form) \\[2pt]
Candidate acceptance rule & first improving exchange & best among $\widetilde{\A}^{(1)},\ldots,\widetilde{\A}^{(k_{\max})}$ \\[2pt]
Guarantees without RIP & not applicable & Theorems~\ref{thm:support_recovery_relax}, \ref{thm:convergence_rate_relax} \\
\bottomrule
\end{tabular}
\end{table}

The first three rows reflect a methodological generalization: deriving a closed-form relevance score for a nonlinear $f$ is nontrivial, and our choice in~\eqref{eq:sacrifice} incorporates both zeroth-order and first-order information of $f$ while remaining cheap to compute. The fourth row is central to our theoretical analysis: by selecting the best candidate rather than 
the first improving one, SCOPE guarantees a larger per-iteration 
decrease in the objective, which is what enables the linear 
convergence rate in Theorems~\ref{thm:convergence-rate} 
and~\ref{thm:convergence_rate_relax}. The last row highlights that the theoretical analysis of \citet{abess} does not carry over: we develop new lemmas (see Section~\ref{sec:theory} and Appendix~\ref{sec:proof-rip-free}) to obtain RIP-free guarantees for general objectives.

\section{Algorithmic Properties}
\label{sec:theory}

This section will present solution guarantees and convergence analysis for the algorithm in Sections~\ref{sec:statistical-theory} and~\ref{sec:computation-theory}. Before presenting these guarantees, we show assumptions and necessary discussions in Section~\ref{sec:assumptions}. To avoid the RIP-type condition mentioned in Section~\ref{sec:assumptions}, we relax the sparsity level and attain related properties in Section~\ref{sec:properties-without-rip}.

\subsection{Assumptions}\label{sec:assumptions}

Our first assumption is related to the convexity and smoothness of the objective function in the subspace. Before the first assumption, we introduce the concepts of restricted strong convexity (RSC) and restricted strong smoothness (RSS) below. 
\begin{definition}[Restricted Strong Convexity and Restricted Strong Smoothness]
	Let $t$ be a positive integer,  
	we say function $f(\cdot)$ is $m_t$-RSC if  
	$$\frac{m_t}{2}\|\bm{x}-\bm{y}\|^{2}_2 + \langle\nabla f(\bm{y}), \bm{x}-\bm{y}\rangle \leq f(\bm{x})-f(\bm{y}),$$
	for any $\bm{x}, \bm{y} \in \mathbb{R}^p$ satisfying $\|\bm{x}-\bm{y}\|_0 \leq t$.
	And say $f(\cdot)$ is $M_t$-RSS if
	$$f(\bm{x})-f(\bm{y}) \leq \frac{M_t}{2}\|\bm{x}-\bm{y}\|^{2}_2 + \langle\nabla f(\bm{y}), \bm{x}-\bm{y}\rangle,$$
	for any $\bm{x}, \bm{y} \in \mathbb{R}^p$ satisfying $\|\bm{x}-\bm{y}\|_0 \leq t$.
\end{definition}
In a geometric sense, a $m_t$-RSC and $M_t$-RSS function $f(\cdot)$ possess strong convexity and strong smoothness in any low-dimensional subspace with a dimension less than $t$. To rephrase, $m_t$-RSC and $M_t$-RSS bound the curvature of $f$ in low-dimensional subspaces.
With the definition of RSC and RSS, the first assumption is given as follows:
\begin{assumption}\label{con:convex-smooth}
	$f(\cdot)$ is $m_{3s}$-RSC and $M_{3s}$-RSS.
\end{assumption}
This assumption only requires $f(\cdot)$ to be strong convex and smooth in subspaces, which is much weaker than the convexity and smoothness in the full space. 
Thus, in Algorithm~\ref{algo:main}, the optimization problem in subspaces (Lines 1 and 7 in Algorithm~\ref{algo:main})
always has a unique minimizer, ensuring that Algorithm~\ref{algo:main} can properly work.
More importantly, it is critical for theoretical analysis because it paves the way for bounding $f(\bm{\theta}) - f(\bm{\theta}^*)$, where $\bm{\theta}$ is a $s$-sparse solution whose support set may be incorrect. 
Actually, Assumption~\ref{con:convex-smooth} is widely used in literature for the studies of algorithms for general sparse-constrained optimization problems \citep{jain2014iterative,yuan2017gradient, zhou_global_nodate}. 
This concept is similar to the restricted isometry property in the field of compressed sensing \citep{1542412}, where a quadratic objective function is involved. Assumption~\ref{con:convex-smooth} serves as its extension for adaptation of general objective functions.

Next, we turn to the second assumption. For unconstrained optimization, $\nabla f(\bm{\theta}^*) = \mathbf{0}$ is a necessary condition for $\bm{\theta}^*$ to be a local minimizer of $f(\cdot)$. Here, we relax the zero gradient condition for $\bm{\theta}^*$ to a weaker one. Let $\vartheta \coloneqq \min\limits_{j \in \mathcal{A}^{*}} \vert \param_{j}^{*}\vert $, and our assumption is given as follows:
\begin{assumption}
	\label{con:bound-gradient} 
	$\| \nabla f (\bm{\theta}^{*} ) \|_{\infty} 
	\leq \frac{0.35}{\sqrt{s}}(1.49 m_{3s} - M_{3s}) \vartheta$.
\end{assumption}
$\| \nabla f (\bm{\theta}^{*}) \|_\infty$ can be interpreted as the assumption on noise (see Section~\ref{sec:application} for concrete examples). From intuition, the appearance of noise may submerge the true signal. Assumption~\ref{con:bound-gradient} expresses this intuitive idea by controlling $\nabla f (\bm{\theta}^{*})$ with the minimum signal of parameters. Such an assumption makes our analyses are applicable to the noised systems that frequently appear in the field of data science. And thus, similar assumptions for $\| \nabla f (\bm{\theta}^{*} ) \|_{\infty}$ are widely imposed for sparsity constraint optimization \citep{bahmani_greedy_2011, yuan2017gradient}. 
Notice that the upper bound in this assumption implicitly comprises two restrictions for $f(\cdot)$. First, it requires~$\|\nabla f(\param^*)\|_{\infty}$ in Assumption \ref{con:bound-gradient} cannot grow too fast with respect to $s$; more precisely, it should be $\mathcal{O}(s^{-\frac{1}{2}})$ which has the same order as \citet{yuan2017gradient}. Second, it requires the so-called restrict condition number $\frac{M_{3s}}{m_{3s}}$ to be upper bounded, which is essentially a RIP-type condition \citep{candes2006stable, needell2009cosamp}. We would like to note that $\frac{M_{3s}}{m_{3s}} \leq 1.49$ is weaker than the prior RIP-type conditions appeared in literature \citep{needell2009cosamp, bahmani_greedy_2011, yuan2017gradient}. Finally, it is worthy to emphasize that, in Section~\ref{sec:properties-without-rip}, we will relax the RIP-type condition and obtain similar properties in sparse solution and computation.

\subsection{Guarantees on Solution}\label{sec:statistical-theory}
The following theorem is one of the main results in this paper:
\begin{theorem}\label{thm:recovery}
	Suppose Algorithm~\ref{algo:main} returns $(\widehat{\bm{\theta}},\widehat{\mathcal{A}})$, then $\widehat{\mathcal{A}}= \mathcal{A}^*$ under Assumptions~\ref{con:convex-smooth}-\ref{con:bound-gradient}.
\end{theorem}
Theorem~\ref{thm:recovery} ensures SCOPE exactly recovers the true support set, which has been pursued for a long time in literature \citep{liu2014forward, yuan2017gradient,yuan2020dual, zhou_global_nodate}. Compared with the algorithms in the literature, the core advantage of SCOPE is that it guarantees to recover the true support set so long as the objective function enjoys certain assumptions. On the contrary, the success of the other algorithms requires tuning parameters such as the step size \citep{yuan2017gradient,yuan2020dual} and the convergence criteria are properly set \citep{liu2014forward}. Unfortunately, the tuning parameters may not be easily determined in practice. For instance, the suitable step size in \citet{yuan2017gradient, yuan2020dual} should be less than $M_{2s}^{-1}$ albeit it is not too intensive to compute because we have to consider $\binom{2s}{p}$ subspaces with dimension $2s$. 

The tuning-free property of SCOPE comes from two aspects. First, it utilizes the splicing operation to find the support set that minimizes the objective function. As mentioned in our remark in Section~\ref{sec:algorithm}, the minimizer always exists without tuning any continuous parameter. In contrast, many iterative algorithms find the stable solution of the support sets that are derived from the estimated parameters during iterations. Since the parameter estimation with gradient descent is sensitive to the choice of step size, selecting an adequate step size is crucial for these algorithms to correctly select the true support set (see, e.g., \citet{yuan2017gradient}). 
Second, SCOPE fully leverages the information of $s$. While iterative algorithms that shift the support size during iterations pay additional costs to find the true support set (e.g., \citet{liu2014forward}).


Now, we present a direct result of Theorem~\ref{thm:recovery} that reveals the property of the parameter estimation returned by Algorithm~\ref{algo:main}.
Denote the oracle solution $\widehat{\bm{\theta}^*}$ as 
$
\widehat{\bm{\theta}^*} \coloneqq \mathop{\argmin}\limits_{\operatorname{supp}\{\bm{\theta}\} = \mathcal{A}^*} f(\bm{\theta}),
$
then we have the following results. 
\begin{corollary}\label{coro:oracle}
	Under Assumptions~\ref{con:convex-smooth}-\ref{con:bound-gradient}, $\widehat{\bm{\theta}} = \widehat{\bm{\theta}^*}$. Particularly, if $\| \nabla f (\bm{\theta}^{*}) \|_\infty = 0$, then $\widehat{\bm{\theta}} = \bm{\theta}^*$.
\end{corollary}

\subsection{Convergence Analysis}\label{sec:computation-theory}

This section mainly asserts the computational efficiency of Algorithm~\ref{algo:main}. We first show that Algorithm~\ref{algo:main} has a linear convergence rate.

\begin{theorem}\label{thm:convergence-rate}
	Follow with Assumptions in Theorem~\ref{thm:recovery}, there exists a constant $\delta_1 > 1$ such that
	\begin{equation}\label{eq:convergence-rate}
		\vert  f(\bm{\theta}^{t+1})-f(\bm{\theta}^{*}) \vert 
		\leq \delta_1^{-1} \vert  f(\bm{\theta}^{t})-f(\bm{\theta}^{*}) \vert ,
	\end{equation}
	when $\bm{\theta}^t$ is not the output of Algorithm~\ref{algo:main}. 
\end{theorem}
Theorem~\ref{thm:convergence-rate} claims that, so long as the SCOPE algorithm has not converged, then the splicing operator would give a new solution that substantially decreases the objective function value. From this property, we can deduce SCOPE converges quickly. Though similar properties are achieved in literature~\citep{bahmani_greedy_2011,yuan2017gradient,yuan2020dual}, it is remarkable that our result is obtained without tuning continuous parameters like \citet{bahmani_greedy_2011}. 

Our next theorem gives an explicit \emph{sufficient} number of splicing iterations for support recovery.
\begin{theorem}\label{thm:convergence}
    Suppose the assumptions of Theorem~\ref{thm:recovery} hold. Then 
    there exists a constant $\delta_2>0$ such that, whenever the 
    number of splicing iterations satisfies
    \begin{equation}\label{eq:iteration_times}
        t \;\geq\; \log_{\delta_1}\!\left\lceil
        \frac{\vert f(\bm{\theta}^{0})-f(\bm{\theta}^{*}) \vert}
             {\delta_{2}\, m_{3s}\, \vartheta^{2}}\right\rceil,
    \end{equation}
    Algorithm~\ref{algo:main} recovers the true support, i.e., 
    $\mathcal{A}^{t}=\mathcal{A}^{*}$. Here $\lceil\cdot\rceil$ is 
    the ceiling function, $\delta_1$ is defined in 
    Theorem~\ref{thm:convergence-rate}, and 
    $f(\bm{\theta}^{0}) \coloneqq 
    \min_{\operatorname{supp}\{\bm{\theta}\} = \mathcal{A}^0} 
    f(\bm{\theta})$ for an arbitrary $s$-cardinality subset 
    $\mathcal{A}^0 \subseteq [p]$.
\end{theorem}
Theorem~\ref{thm:convergence} guarantees that SCOPE identifies the 
true active set within a logarithmic number of splicing iterations. 
It also makes explicit how the quality of the initial support 
$\mathcal{A}^0$ (through $f(\param^0) - f(\param^*)$) and the 
minimum signal strength $\vartheta$ affect the iteration count: 
fewer iterations are required when $\mathcal{A}^0$ is close to 
$\mathcal{A}^*$ or when $\vartheta$ is large.

\subsection{Properties without RIP-type Conditions}\label{sec:properties-without-rip}
As discussed in Section~\ref{sec:assumptions}, Assumption~\ref{con:bound-gradient} encompasses a RIP-type condition, which may be restrictive in practical scenarios. However, the following theorems show that appropriately relaxing sparsity levels enables the SCOPE to accurately estimate sparse solutions without assuming the RIP-type condition.

To establish the theorems, we first provide the mathematical characterization of the conditions required in this section. In this part, we denote $s^* \coloneqq | \A^* |$ and let $s \coloneqq | \widehat\A |$. The two conditions are given below. 
\begin{assumption}\label{con:convex-smooth-relaxed}
	$f(\cdot)$ is $m_{s+s^*}$-RSC and $M_{s+s^*}$-RSS.
\end{assumption}
\begin{assumption}\label{con:bound-gradient-relaxed}
	$\|\nabla f(\param^*)\|_{\infty} < \frac{m_{s+s^*}}{2\sqrt{s+s^*}} \vartheta$ where $m_{s+s^*}$ is defined in Assumption~\ref{con:convex-smooth-relaxed}.
\end{assumption}
Assumption~\ref{con:convex-smooth-relaxed} characterizes the relaxing sparsity levels mentioned earlier, requiring $f(\param)$ to exhibit strong convexity and strong smoothness on any $(s+s^*)$-dimensional subspace. The second assumption relates to the systematic noise, ensuring that the noise is controlled by the minimal magnitude of the sparse signal and decays to zero with a rate of $O(s^{-\frac{1}{2}})$. This rate aligns with Assumption~\ref{con:bound-gradient} and matches the rate in previous studies \citep{jain2014iterative, yuan2017gradient}. More importantly, the constant involved in Assumption~\ref{con:bound-gradient-relaxed} is only proportional to the $m_{s+s^*}$ and is independent to the restricted condition number $\frac{M_{s+s^*}}{m_{s+s^*}}$ in Assumption~\ref{con:bound-gradient}. Consequently, this assumption essentially frees from RIP-type conditions.

We now present the guarantee for the solution when sparsity levels are relaxed. 
\begin{theorem}\label{thm:support_recovery_relax}
	Suppose Assumptions~\ref{con:convex-smooth-relaxed} and~\ref{con:bound-gradient-relaxed} hold when setting $s>(1+2 M_{s+s^*}^2 / m_{s+s^*}^2) s^*$. Then, $(\widehat{\param},\widehat{\mathcal{A}})$ satisfies (i) $\widehat{\mathcal{A}}\supseteq \mathcal{A}^*$ and (ii) $\supp(\mathcal{H}_{s^*}(\widehat{\param})) = \A^*$.
\end{theorem}
From Theorem~\ref{thm:support_recovery_relax}, the SCOPE gives an estimated support set that includes the true support set. Even more interestingly, (ii) reveals the top $s^*$ coordinates (in the sense of magnitude) in $\widehat{\param}$ are exactly the true support set. It is worth noting that while the sparsity level shall increase linearly with respect to the restricted condition number, the rate of increase is slower compared to \citet{jain2014iterative} and \citet{yuan2017gradient}. This advantage arises from the fact that, in each splicing iteration, Algorithm~\ref{algo:main} effectively utilizes the objective value to identify an appropriate splicing set (see Lemma~\ref{lemma:loss_down_bound}).

Our next theorem demonstrates that even without imposing RIP-type conditions, Algorithm~\ref{algo:main} still exhibits a linear convergence rate when the sparsity level is relaxed.
\begin{theorem}\label{thm:convergence_rate_relax}
	Suppose Assumption~\ref{con:convex-smooth-relaxed} holds when setting $s>2(1+2 \frac{M_{s+s^*}^2}{m_{s+s^*}^2}) s^*$ in Algorithm~\ref{algo:main}, then
	$$f(\param^{t+1})-f(\param^*) \leq \left(1-\delta_4\right)\left(f(\param^t)-f(\param^*)\right),$$
	when $f(\param^t) \geq f(\param^*)$ and $\param^t$ is not the solution of Algorithm~\ref{algo:main},
        where $\delta_4=\frac{m_{s+s^*}}{4 M_{s+s^*}}\left(1-\frac{4 s^*}{s-2 s^*} \frac{M_{s+s^*}^2}{m_{s+s^*}^2}\right) \in (0, \frac{m_{s+s^*}}{4 M_{s+s^*}})$.
\end{theorem}
In comparison to Theorem~\ref{thm:convergence-rate}, Theorem~\ref{thm:convergence_rate_relax} does not require the constraint on $\|\nabla f(\param^*)\|_{\infty}$. This is because Assumption~\ref{con:convex-smooth-relaxed} imposes an RSC and RSS assumption on higher dimensional subspaces, which simplifies the analysis for the deduction of the objective value after each splicing iteration.
%

With the linear convergence rate established in Theorem~\ref{thm:convergence_rate_relax}, we can derive the number of splicing iterations for including $\A^*$ when the sparsity level is relaxed.
\begin{theorem}\label{thm:convergence_relax}
    Under the same conditions and notation as 
    Theorem~\ref{thm:convergence_rate_relax}, and if 
    Assumption~\ref{con:bound-gradient-relaxed} additionally holds, 
    then whenever the number of splicing iterations satisfies
    $$
        t \;\geq\; \log_{(1-\delta_4)^{-1}}\!\left\lceil
        \frac{\max\{f(\param^0)-f(\param^*),\,0\}}
             {\left(\tfrac{1}{2}-\delta_5\right) m_{s+s^*}\, \vartheta^{2}}
        \right\rceil,
    $$
    we have $\mathcal{A}^{t}\supseteq\mathcal{A}^{*}$, where 
    $\delta_5$ is a positive constant specified in the proof.
\end{theorem}

\section{Applications}\label{sec:application}

We now instantiate Algorithm~\ref{algo:main} on three canonical 
sparse learning problems: compressed sensing 
(Section~\ref{sec:linear}), sparse classification via $\ell_1$-free 
logistic regression (Section~\ref{sec:classification}), and 
reconstructing binary pairwise Markov networks 
(Section~\ref{sec:ising}). For each task we identify standard 
statistical conditions on the design and the noise---the sparse 
Riesz condition, sub-Gaussian noise, and the Ising 
identifiability condition, respectively---under which 
Assumptions~\ref{con:convex-smooth} and~\ref{con:bound-gradient} 
hold with high probability, and we state the resulting support 
recovery and convergence guarantees as corollaries of 
Theorems~\ref{thm:recovery} and~\ref{thm:convergence}. 
Analogous guarantees under 
Assumptions~\ref{con:convex-smooth-relaxed} 
and~\ref{con:bound-gradient-relaxed} can be derived by the same 
argument with the sparsity level $3s$ replaced by $s + s^*$; we 
omit these derivations for brevity.

\subsection{Compressed Sensing}\label{sec:linear}

Standard compressed sensing problems aim to estimate the 
a $s$-sparse vector $\param^*$ from noisy linear measurements $\mathbf{y}$ that comes from the linear model: 
$\mathbf{y}=\mathbf{X} \param^* +\bm{\epsilon},$
where $\mathbf{X} \in \mathbb{R}^{n\times p}$ is a known sensing matrix, $\bm{\epsilon} = (\epsilon_1, \ldots, \epsilon_n)^\top$ is the additive measurement noise, and $\epsilon_1, \ldots, \epsilon_n$ are \textit{i.i.d.} zero-mean random noises.
Under the linear model, compressed sensing is formulated as a sparsity-constraint problem
~\eqref{eq:quardticobj} in which the objective function $f(\param) = \|\mathbf{y}-\mathbf{X}\param\|^{2}$.
Since $f$ is a convex function with twice differentiation, we can directly apply Algorithm~\ref{algo:main} to obtain an estimator $(\widehat{\param}, \widehat{\mathcal{A}})$ for $(\param^*, \mathcal{A}^*)$. 

We now verify Assumptions~\ref{con:convex-smooth} 
and~\ref{con:bound-gradient} for this task under standard 
conditions on the design matrix and the noise. Theorems~\ref{thm:recovery}--\ref{thm:convergence} 
then apply directly, giving high-probability support recovery and 
a linear rate of convergence for the resulting estimator. The two 
conditions are the following.
\begin{enumerate}[label=(A\arabic*), start=1]
	\item There exists constants $0<m_{3s}<M_{3s} \leq 1.49m_{3s} <\infty$ such that $\forall \mathcal{A} \subseteq \{1,2, \dots, p\}$ with $\vert \mathcal{A}\vert  \leq 3s$ and $\forall \bm{u} \in \mathbb{R}^{\vert \mathcal{A}\vert }$, we have:
    $m_{3s}\|\bm{u}\|^{2} \leq \frac{1}{n}\left\|\mathbf{X}_{\mathcal{A}} \bm{u}\right\|^{2} \leq M_{3 s}\|\bm{u}\|^{2}$.  \label{con:linear:correlation}
	\item For each $i \in [n]$, $\epsilon_{i}$ is a sub-Gaussian random variable with parameter $\sigma$, i.e., for all $t \geq 0$, we have:
	$\mathbb{P}\{\left\vert \epsilon_{i}\right\vert  \geq t\} \leq 2 \exp \left(-t^{2} / \sigma^{2}\right).$
	\label{con:linear:noise}
\end{enumerate}
Condition~\ref{con:linear:correlation} refers to the sparse Riesz condition in the literature for the analysis of sparse generalized linear models, which serves as an identifiability condition for support recovery \citep{zhang2008sparsity, shen2012likelihood}. Besides, as we have discussed in Section~\ref{sec:assumptions}, $M_{3s} \leq 1.49 m_{3s}$ is the widely accepted RIP-type condition for learning sparse linear model. Condition~\ref{con:linear:noise} is a standard assumption in literature \citep{wainwright2019hds}. It ensures that the probability of deviation from zero decays exponentially fast. 

We begin to confirm Assumptions~\ref{con:convex-smooth} and~\ref{con:bound-gradient} hold wherein we assume $\mathbf{x}_i^\top \mathbf{x}_i=n$ for every $i \in [p]$ to facilitate the analyses. 
First, for any $\param, \param^\prime$ such that $\vert \operatorname{supp}(\param - \param^\prime)\vert  \leq 3s$, we have:
$$
\frac{m_{3 s}}{2}\|\param'-\param\|^{2}\leq
f(\param')-f(\param)-\nabla f(\param)^\top (\param'-\param)
=\frac{1}{2n}\|\mathbf{X}(\param'-\param)\|^{2} \leq \frac{M_{3 s}}{2}\|\param'-\param\|^{2}
$$
due to Condition~\ref{con:linear:correlation}. 
Therefore, Assumption~\ref{con:convex-smooth} holds. 
Second, we have
$
n \|\nabla f(\param^{*}) \|_{\infty}= \|\mathbf{X}^{\top} \bm{\epsilon} \|_{\infty}= \max\limits_{1 \leq i \leq p}\vert \mathbf{X}_{i}^{\top} \bm{\epsilon}\vert .
$
Owing to Condition~\ref{con:linear:noise}, $\mathbf{X}_{i}^{\top} \bm{\epsilon}$ is a sub-Gaussian random variable with parameter $\sqrt{n}\sigma$ \citep{wainwright2019hds}. Then, simple algebra can show that for the constant $c = 1.49m_{3s} - M_{3s} \geq 0$,
\begin{align*}
    \mathbb{P}\left(\left\|\nabla f\left(\param^{*}\right)\right\|_{\infty} \geq c \vartheta \right)
    = \mathbb{P}\left(n^{-1} \max_{i \in [p]} \vert \mathbf{x}_i^\top \bm{\epsilon}\vert  \geq c\vartheta \right) 
    \leq p \mathbb{P}\left(\vert \mathbf{x}_i^\top \bm{\epsilon}\vert  \geq n c\vartheta \right) 
    \leq 2p \exp\left(-\frac{nc^2\vartheta^2}{2 \sigma^{2}}\right).
\end{align*}
Consequently, Assumption~\ref{con:bound-gradient} holds with probability at least $1-2 \exp\{\log{p}-\frac{n c^{2} \vartheta^{2}}{2 \sigma^{2}} \}$.

In summary, when Conditions~\ref{con:linear:correlation}-\ref{con:linear:noise}, Theorem~\ref{thm:recovery} ensures $\operatorname{supp}(\widehat{\mathcal{A}}) = \operatorname{supp}(\mathcal{A}^*)$ with high probability. In particular, if $\vartheta = O(\sqrt{\log{p} / n})$, the probability of $\operatorname{supp}(\widehat{\mathcal{A}}) = \operatorname{supp}(\mathcal{A}^*)$ converges to 1 as $n \rightarrow \infty$. Furthermore, Theorem~\ref{thm:convergence} shows that, with high probability, Algorithm~\ref{algo:main} terminate after $O\left(\log_{\delta_1}\left(\frac{\vert f(\param^0) - \|\bm{\epsilon}\| \vert }{\delta_2 m_{3s} \vartheta} \right)\right)$ splicing iteration. Since each splicing iteration takes $O(snp)$ time complexity, the total time complexity of Algorithm~\ref{algo:main} is $O\left(snp\left(\log_{\delta_1}\left(\frac{\vert f(\param^0) - \|\bm{\epsilon}\| \vert }{\delta_2 m_{3s} \vartheta} \right)\right)\right)$.

\subsection{Sparse Classifier}\label{sec:classification}
Finding the core elements to classify the objects is a crucial application task and attracts much attention from the statistical learning community. The logistic regression model is an important model to solve this problem, which leverages the information of explanatory variables $\mathbf{x} \in \mathbb{R}^p$ to predict the class of response variable $y \in \{0, 1\}$. The underlying logistic regression model is expressed as follows:
\begin{equation}\label{eq:logistic-regression}
	\mathbb{P}(y = 1 \mid \mathbf{x}) = \frac{\exp(\mathbf{x}^\top \param^*)}{1+\exp(\mathbf{x}^\top \param^*)},
\end{equation}
where $\param^* \in \mathbb{R}^p$ is an unknown $s$-sparse vector that to be estimated. 

We want to estimate $\param^*$ by collecting $n$ independent samples of the explanatory variables and the response variable that are stored into $\mathbf{X} = (\mathbf{x}_1, \ldots, \mathbf{x}_n)^\top \in \mathbb{R}^{n\times p}$ and $\mathbf{y} = (y_1, \ldots, y_n) \in \mathbb{R}$. With these samples, we estimate $\param^*$ by minimizing the negative log-likelihood under sparsity constraint $\argmin_{\param \in \mathbb{R}^p}, \loss(\param) \textup{ s.t. } \| \param \|_0 \leq s$,
where 
$$f(\param) \coloneqq -\frac{1}{n}\sum_{i=1}^{n}\left\{\bm{y}_{i} \mathbf{x}_{i}^{\top} \param-b(\mathbf{x}_{i}^{\top} \param)\right\}$$
and $b(t)=\log \left(1+\exp(t)\right)$. We can apply Algorithm~\ref{algo:main} on this task to solve $(\widehat{\param}, \widehat{\mathcal{A}})$ because $f(\cdot)$ has the second differentiation.

Next, we show that, under some reasonable conditions, Algorithm~\ref{algo:main} completes iterations quickly, and its output enjoys statistical guarantees. The conditions are presented below.
\begin{enumerate}[label=(B\arabic*), start=1]
	\item The same as Condition~\ref{con:linear:correlation} in Section~\ref{sec:linear}.\label{con:classification:correlation}
 	\item There exists a positive constant $\delta$ satisfying $\frac{M_{3s}}{5.96m_{3s}} \leq \delta \leq \min\limits_{i \in [n]} \nabla^2 b(\mathbf{x}_i \check{\param})$ for any $3s$-sparse vector $\check{\param}$ such that $\loss(\check{\param}) < \loss(\param^0)$.
	\label{con:classification:noise}
\end{enumerate}
The upper bound for $\delta$ in Condition~\ref{con:classification:noise} restricts $\mathbf{y}_i$ cannot be a degenerate random variable to avoid infinitely small variances \citep{rigollet2012kullback} so as to ensure the objective function being strongly convex. This assumption is also introduced in \citet{yuan2017gradient, yuan2020dual}. The lower bound condition on $\delta$ is a RIP-type condition for $\frac{M_{3s}}{m_{3s}}$ under logistic regression.

We verify Assumptions~\ref{con:convex-smooth}-\ref{con:bound-gradient} with Conditions~\ref{con:classification:correlation}-\ref{con:classification:noise}. Without loss of generality, we assume $\| \mathbf{x}_j \| = \sqrt{n}$.
First, we verify RSC and RSS of $f(\cdot)$.
According to Taylor-series expansion, for $\forall \param,\param^{\prime}$ with $\left\|\param-\param^{\prime}\right\|_{0} \leq 3s$, we have:
$f(\param^{\prime})-f(\param)-\nabla f(\param)^{\top}(\param^{\prime}-\param )=\frac{1}{2 n} \| H(\mathbf{X}, \bar{\param}) \mathbf{X}(\param^{\prime}-\param )\|^{2},$
where $H(\mathbf{X}, \bar{\param})= \operatorname{diag}\left\{\sqrt{\nabla^{2} b(\mathbf{x}_{1}^{\top} \bar{\param})}, \cdots, \sqrt{\nabla^{2} b(\mathbf{x}_{n}^{\top} \bar{\param})}\right\}$ and $\bar{\param}=\lambda \param+(1-\lambda) \param^{\prime}$ for some $\lambda \in(0,1)$. 
Then, by Conditions~\ref{con:classification:correlation} and~\ref{con:classification:noise} together with the fact that $\nabla^{2} b(t) \leq \frac{1}{4}$, we have:
\begin{align*}
    \frac{1}{2 n}\left\|H(\mathbf{X}, \bar{\param}) \mathbf{X}(\param^{\prime}-\param )\right\|^{2} 
    &\leq \frac{1}{8 n}\left\|\mathbf{X}(\param^{\prime}-\param )\right\|^{2} \leq \frac{M_{3 s}}{8}\left\|\param^{\prime}-\param\right\|^{2}, 
    \\
    \frac{1}{2 n}\left\|H(\mathbf{X}, \bar{\param}) \mathbf{X}(\param^{\prime}-\param )\right\|^{2} 
    &\geq \frac{\delta}{2 n}\left\|\mathbf{X}(\param^{\prime}-\param )\right\|^{2} \geq \frac{\delta m_{3s}}{2}\left\|\param^{\prime}-\param\right\|^{2}.
\end{align*}
Thus, $f(\param)$ is $\frac{M_{3 s}}{4}$-RSS and 
$\frac{m_{3s}}{5.96}$-RSC. 
Next, for the gradient at $\param^*$, it has the form $\nabla f(\param^{*})= (g_{1}, \ldots, g_{p})$, where 
$
g_{j}=\frac{1}{n}\sum_{i = 1}^{n}(\nabla b(\mathbf{x}_{i}^{\top} \param^*)-\bm{y}_{i}) \mathbf{x}_{ij}.
$
Since $(\nabla b(\mathbf{x}_{i}^{\top} \param^*)-\bm{y}_{i}) \mathbf{x}_{ij}$ is a zero-mean random variable with range $[-\mathbf{x}_{ij}, \mathbf{x}_{ij}]$, then $g_j$ satisfies 
$\mathbb{P}(\vert  g_j\vert  > c\vartheta) \leq 2 \exp\{-\frac{c^2\vartheta^2 n}{2}\}$ according to Hoeffding's inequality for general bounded random variables where 
$c = \frac{0.35}{\sqrt{s}}(1.49\delta m_{3s} - \frac{M_{3 s}}{4}) > 0$ 
(because of Condition~\ref{con:classification:correlation}).
Consequently, it results in $\mathbb{P}\left(\left\|\nabla f\left(\param^{*}\right)\right\|_{\infty} \geq c\vartheta \right) \leq 2 \exp \left\{\log p-\frac{n c^2\vartheta^{2}}{2}\right\}$, which means Assumption~\ref{con:bound-gradient} holds with high probability. And hence, with Conditions~\ref{con:classification:correlation}--\ref{con:classification:noise}, we can deduce $\operatorname{supp}\{\widehat{\param}\} = \mathcal{A}^*$ holds with probability at least $1 - 2 \exp \left\{\log p-\frac{nc^2 \vartheta^{2}}{2}\right\}$. When $\vartheta = O(\sqrt{\log{p}/n})$,  we can see that $\mathbb{P}(\operatorname{supp}\{\widehat{\param}\} = \mathcal{A}^*)$ reaches 1 as $n$ goes to infinity.

\subsection{Sparse Ising model}\label{sec:ising}
The Ising model is a classical undirected graphical model for pairwise interactions among binary random variables, with applications in statistical physics, neuroscience, computational biology, and the social sciences \citep{lokhov2018optimal}. Let $\mathbf{x} = (x_1, \ldots, x_p)^\top \in \{-1, 1\}^p$. The Ising model with interaction matrix $\bm{\theta}^* \in \mathbb{R}^{p \times p}$ 
assigns probability
\begin{equation}\label{eq:ising-density}
    \mathbb{P}(\mathbf{x}; \bm{\theta}^*) 
    \;=\; \frac{1}{Z(\bm{\theta}^*)} 
    \exp\!\left\{ \tfrac{1}{2} \sum_{k, l = 1}^{p} 
    \bm{\theta}^*_{kl}\, x_k x_l \right\},
    \qquad
    Z(\bm{\theta}^*) \;\coloneqq\;
    \sum_{\mathbf{z} \in \{-1, 1\}^p}
    \exp\!\left\{ \tfrac{1}{2} \sum_{k, l = 1}^{p} 
    \bm{\theta}^*_{kl}\, z_k z_l \right\},
\end{equation}
where $\bm{\theta}^*$ is symmetric with zero diagonal, and its 
off-diagonal entries encode the strength of pairwise interactions. 
Variables $k$ and $l$ are adjacent in the underlying interaction 
graph if and only if $\bm{\theta}^*_{kl} \neq 0$; recovering 
$\operatorname{supp}(\bm{\theta}^*)$ is therefore equivalent to 
recovering the graph structure. In most applications of interest 
each variable interacts with only a few others, so $\bm{\theta}^*$ 
is sparse and our sparsity-constrained framework applies. It is worth noting that, in the Ising model, variables $k$ and $l$ are connected by an edge if $\param^*_{k l} \neq 0$. Therefore, $\param^*$ corresponds to an interaction network for variables $x_1, \ldots, x_p$. From intuition, each variable in the interaction network only connects to a few variables; thus, the interaction network should be sparse. To rephrase, $\param^*$ is a sparse matrix to be recovered. 

To this end, given $n$ samples $\mathbf{X}=(\mathbf{x}_{1}, \ldots, \mathbf{x}_{n})^\top \in \{-1,1\}^{n\times p}$, we can minimize the negative log-likelihood with the constraint of non-zero entries in the upper diagonal matrix:
\begin{align*}
	\argmin_{\param \in \mathcal{S}_p} -\frac{1}{n}\sum_{i=1}^n \log\left(\mathbb{P}(\mathbf{x}_i)\right), \textup{ s.t. } \|\param\|_0 \leq s,
\end{align*}
where $\mathcal{S}_p$ is the space of $p$-by-$p$ zero-diagonal symmetric matrix and $\|\param\|_0 = \sum_{k < l}\textup{I}(\param_{k l} \neq 0)$. Computing the log-likelihood is intractable because evaluating 
$Z(\bm{\theta})$ requires summing over the $2^p$ configurations 
of $\mathbf{x}$. We therefore adopt the pseudo-likelihood as a 
tractable surrogate:
\begin{align*}
    \argmin_{\param \in \mathcal{S}_p} f(\param) \coloneqq 
    -\frac{1}{n}\sum_{i=1}^n \log\left[ \prod_{k =1}^p\mathbb{P}(\mathbf{x}_{i k} \vert  \mathbf{x}_{i 1}, \ldots, \mathbf{x}_{i k - 1}, \mathbf{x}_{i k + 1}, \ldots, \mathbf{x}_{i p})\right], \textup{ s.t. } \|\param\|_0 \leq s.
\end{align*}
The pseudo-likelihood has been widely used for learning Ising 
models \citep{hoing_estimation_nodate, xue2012nonconcave} because 
it is (i) computable in $O(np)$ time per evaluation, (ii) twice 
continuously differentiable and convex, and (iii) a consistent 
surrogate for the likelihood when the underlying graph is sparse..  
With these advantages, Algorithm~\ref{algo:main} can be efficiently conducted upon the pseudo-likelihood and returns $(\widehat{\param}, \widehat{\mathcal{A}})$. 

To assure the theoretical guarantees of Algorithm~\ref{algo:main}, we certify Assumptions in Section~\ref{sec:assumptions} hold under following conditions:
\begin{enumerate}[label=(C\arabic*), start=1]
	\item For $\forall \mathcal{A} \subseteq [p]$ with $\vert \mathcal{A}\vert  \leq \min\{3s, p-1\}$, there exist constants $0<m_{3s}<M_{3s}<\infty$ such that $n m_{3s}\|\bm{u}\|^{2} \leq \left\|\mathbf{X}_{\mathcal{A}} \bm{u}\right\|^{2} \leq n M_{3s}\|\bm{u}\|^{2},$
	for $\forall \bm{u} \in \mathbb{R}^{\vert \mathcal{A}\vert }$.
	\label{con:ising:correlation}
	\item There exists a constant $\omega > 0$ such that every $3s$-sparse 
$\check{\bm{\theta}}$ satisfying $f(\check{\bm{\theta}}) < 
f(\bm{\theta}_0)$ obeys the row-wise $\ell_1$ bound
\begin{equation*}
    \max_{l \in [p]} \sum_{k \neq l} |\check{\bm{\theta}}_{kl}| 
    \;\leq\; \omega,
\end{equation*}
and, writing $\Delta \coloneqq e^{2\omega} / (e^{2\omega} + 1)^2$, 
the restricted condition number satisfies 
$(1 + M_{3s}) / \bigl[11.92 \, (1 + m_{3s})\bigr] < \Delta$.\label{con:ising:parameter}
\end{enumerate}
When $3s < p-1$, Condition~\ref{con:ising:correlation} is identical to Condition~\ref{con:classification:correlation} and can be interpreted similarly. This condition also implies that when $p > n$, $s$ cannot exceed $O(p)$, which is reasonable because the collected data is limited. On the other hand, if $n > p$, Condition~\ref{con:ising:correlation} can hold even when $\vert \mathcal{A} \vert  = p-1$ once the correlation among variables is bounded. The upper bound for $\max\limits_{l \in [p]} \sum\limits_{k \neq l}\left\vert \check{\param}_{k l}\right\vert$ in Condition~\ref{con:ising:parameter} is to avoid the case when one binary variable can be perfectly predicted, at which randomness does not exist. This upper bound is widely imposed to ensure the identifiability of the sparse Ising model  \citep{santhanam2012information, lokhov2018optimal}. The inequality $(1 + M_{3s}) / [11.92(1 + m_{3s})] < \Delta$ is the 
Ising-specific analogue of the RIP-type condition in 
Assumption~\ref{con:bound-gradient}.

The following Proposition shows that $f(\cdot)$ satisfies Assumption~\ref{con:convex-smooth}. 
\begin{proposition}\label{prop:ising-rsc-rss}
	Under Conditions~\ref{con:ising:correlation} and~\ref{con:ising:parameter}, $f(\cdot)$ is $8\Delta(1+m_{3s})$-RSC and $(1+M_{3s})$-RSS.
\end{proposition}

\begin{remark}
The constant $11.92$ in Condition~\ref{con:ising:parameter} arises 
as $8 \times 1.49$: the factor $8$ comes from the RSC constant in 
Proposition~\ref{prop:ising-rsc-rss}, and the factor $1.49$ 
matches the restricted-condition-number threshold in 
Assumption~\ref{con:bound-gradient}. Combining the two constants 
gives the Ising-specific RIP-type bound.
\end{remark}
Then, we give an upper bound for $\| \nablaf(\param^*) \|_{\infty} \coloneqq \max\limits_{k, l \in [p]} (\nabla f)_{kl}$. For $\forall k , l \in [p]$ satisfying $k \neq l$, we have $-2 \mathbf{x}_{i k} \mathbf{x}_{i l} (2-\phi_{k}(\mathbf{x}_i) - \phi_{l}(\mathbf{x}_i)) \in [-4,4]$ for each $i$. Thus, its independent empirical mean $(\nablaf(\param^*))_{k l}$ is a sub-Gaussian variable with parameter $\frac{4}{\sqrt{n}}$ by Hoeffding's inequality. According to Condition~\ref{con:ising:parameter}, we can define a positive value $c=\frac{0.35}{\sqrt{s}}(11.92\Delta(1+m_{3s}) - (1 + M_{3s}))$, then owing to $\mathbb{E} \left[\nablaf(\param^*)\right] = \mathbf{0}$ \citep{xue2012nonconcave}, we have $\mathbb{P} \left(\left\|\nabla f(\param^*)\right\|_{\infty} \geq c\vartheta \right) \leq 2 \exp \left\{\log \frac{p(p-1)}{2}-\frac{n c^2\vartheta^{2}}{8}\right\}$, implying Assumptions~\ref{con:bound-gradient} holds with high probability. 
Particularly, when $\vartheta = O(\sqrt{\log{p}/n})$, Assumption~\ref{con:bound-gradient} holds with probability 1 as $n \to \infty$; subsequently, the properties of SCOPE hold with probability 1. It means that SCOPE can exactly recover the underlying graph structure of the Ising model with probability one within a few iterations.

\section{Numerical Experiments}\label{sec:experiments}

In this section, we evaluate the empirical performance of the SCOPE algorithm with two numerical experiments. The first experiment aims to illustrate the support recovery property and computational merits of our algorithm by comparing it with the exact solver (see Section~\ref{sec:experiment-exact-solver}). The second experiment demonstrates the empirical advantages of our method against the state-of-the-art algorithms (see Section~\ref{sec:experiment-sota}). Additional experiments regarding to $k_{\max}$ is defer to Appendix~\ref{sec:max-splicing-size}.

\subsection{Comparisons with the exact solver}\label{sec:experiment-exact-solver}
Here, we select a commercial solver, GUROBI with version 12.0.0 \citep{gurobi}, for comparison because it is one of the most popular solvers for sparse-constrained optimization. We will compare SCOPE with GUROBI on problem~\eqref{eq:quardticobj}. All of experiment are conduct on an Ubuntu platform with Intel(R) Xeon(R) Silver 4210 CPU @ 2.20GHz and 64 RAM.
We select problem~\eqref{eq:quardticobj} as our benchmark because it is well-developed and supported by GUROBI.

We depict the detailed numerical settings for problem~\eqref{eq:quardticobj} in the following. 
Specifically, 
we create noisy linear measurements $\mathbf{y}$ that comes from the underlying linear model with a $s$-sparse parameter vector~$\bm{\theta}^*$:
$\mathbf{y}=\mathbf{X} \bm{\theta}^* +\bm{\epsilon},$
where $\bm{\epsilon} = (\epsilon_1, \ldots, \epsilon_n)^\top$ is the additive measurement noise, and $\epsilon_1, \ldots, \epsilon_n$ are \textit{i.i.d.} zero-mean random noises. 
First, we generate a random predictor matrix $\mathbf{X} \in \mathbb{R}^{n \times p}$ whose row vectors $\mathbf{X}_{1,\cdot}, \ldots \mathbf{X}_{n,\cdot} \overset{i.i.d.}{\sim} \mathcal{N}(\mathbf{0}, \mathbf{\Sigma})$, where the $(i, j)$-entry of covariance matrix is $\mathbf{\Sigma}_{ij}=\rho^{\vert  i-j\vert }$. Here, we set $\rho=0.6$. 
In terms of the $s$-sparse regression coefficients $\bm{\theta}^*$, the indices of its non-zero entries are randomly selected from $[p]$, and their values are selected from $\{-100, 100\}$ with equal chance. As for the random noise $\bm{\epsilon}$, it is sampled from $\mathcal{N}(\mathbf{0}, \sigma^2 \mathrm{I}_{p\times p})$ such that the signal-to-noise ratio $\|\mathbf{X}\bm{\theta}^*\|^2 / \sigma^2$ is equal to $1$. 
Finally, the response is generated according to the above linear model, and it is adjusted to have a zero mean and a variance of one for normalization purposes. 

We measure the empirical performance of algorithms by the proportion of non-zero entries that are correctly selected.
To evaluate the computational performance, we measure each algorithm's running time (measured in seconds).
We investigate both support recovery accuracy and computational behavior in the cases where sample size $n$, dimension $p$, and sparsity $s$ vary, but the remaining are fixed. The results are summarized in  Figure~\ref{fig:compare-exact-solver}.

The left-upper panel of Figure~\ref{fig:compare-exact-solver} shows that both SCOPE and GUROBI can recover the support with perfect accuracy when the sample size is more than 800. 
Besides, we can see that SCOPE and GUROBI have a very close performance. In terms of runtime performance presented in the right-upper panel of Figure~\ref{fig:compare-exact-solver}, the runtime of GUROBI is more than one second in most cases and its runtime would increase when it has a lower probability of identifying the true support set. On the contrary, SCOPE can finish iterations in one-tenth of a second in most cases.  

Next, we turn to the results when $p$ varies, but $n, s$ are fixed (presented in the middle panel of Figure \ref{fig:compare-exact-solver}). We can first see that both SCOPE and GUROBI have a higher support recovery accuracy when $p$ decreases. And they have a very similar performance in most cases. From the right-middle panel of Figure~\ref{fig:compare-exact-solver}, the runtime of GUROBI quickly grows as $p$ increases while the runtime of SCOPE grows slowly. Thus, SCOPE still shows dominant runtime advantages --- it converges in less than one second when $p=100$, but GUROBI has reached its time limit. 
Finally, as we can see in the left-bottom of Figure~\ref{fig:compare-exact-solver}, an increasing $s$ makes the sparsity optimization harder, and thus, both GUROBI and SCOPE have lower accuracy. The SCOPE seems to be better when $s \geq 6$, and it maintains a competitive performance in other cases. In all, the gap between GUROBI and SCOPE is still tiny. The main difference between them lies in the runtime presented in the right-bottom panel of Figure~\ref{fig:compare-exact-solver}. GUROBI has a sharp-growth runtime concerning $s$. As for SCOPE, its runtime grows much slower and can converge in less than one second.

\begin{figure}[t]
\vspace{-10pt}
	\centering 
	\includegraphics[width=1.0\textwidth]{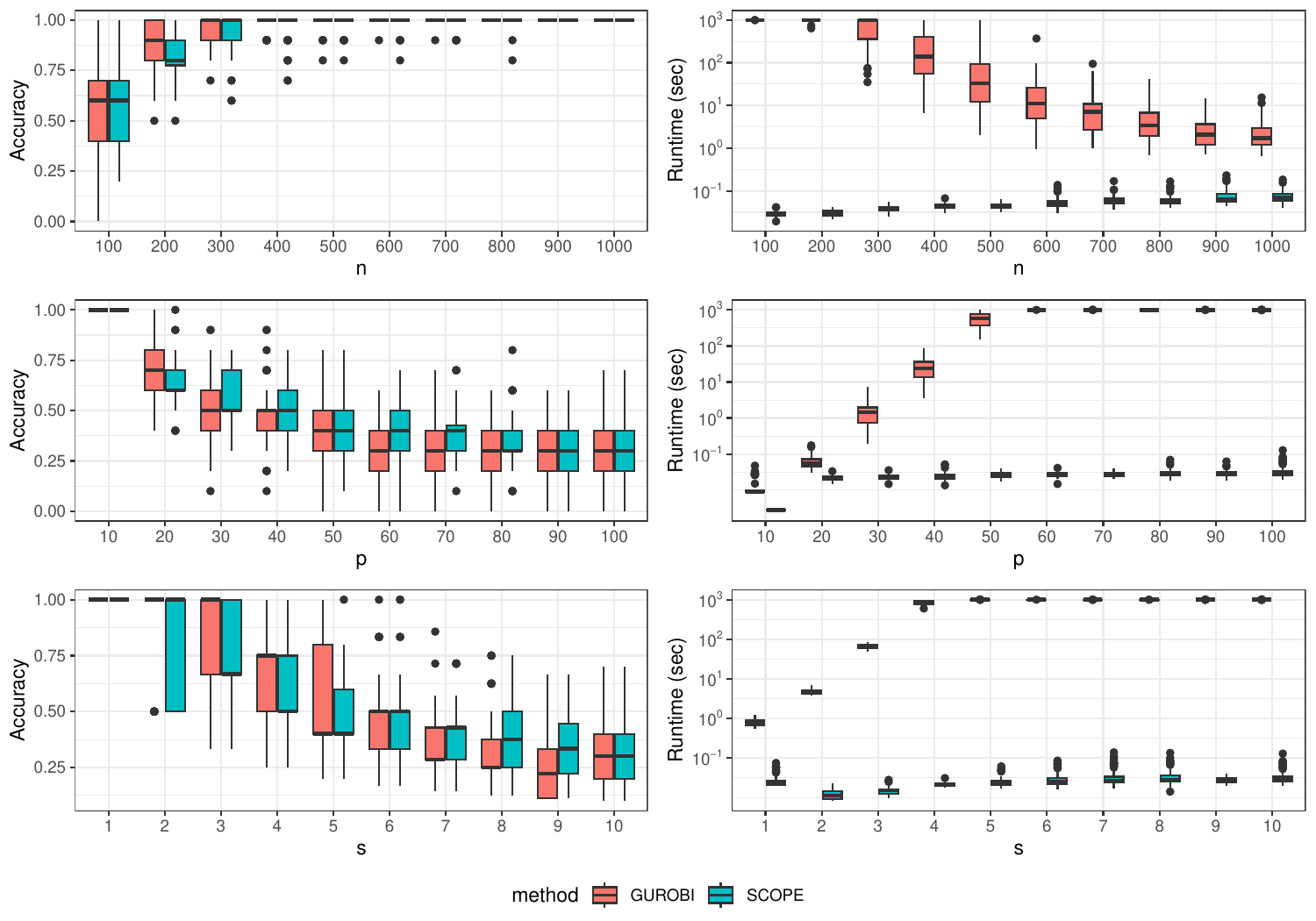}
	\vspace{-24pt}
	\caption{The boxplot of accuracy (Left panel) and runtime (Right panel). Upper panel: $n$ increases when $p$ and $s$ is fixed at $p=100$ and $s=10$; Middle panel: $p$ increases when fixing $n=50$, $s=10$; Bottom panel: increasing $s$ but fixing $n=50$, $p=100$. Note that the runtime of the two methods is limited to 1000 seconds. The experiment was independently repeated 100 times.}\label{fig:compare-exact-solver} 
\vspace{-10pt}
\end{figure}

\subsection{Comparisons with State-of-the-Art Methods}\label{sec:experiment-sota}
Here, we compare our algorithm with the other competitive methods in the literature. We list the competitors below with descriptions for their implementation.
\begin{itemize}[leftmargin=*]
\item Lasso-type regularization method. We use the \textsf{cvxpy} Python interface \citep{diamond2016cvxpy} that wraps a core \textsf{C++} backend to get the solution of the convex relaxation of \eqref{originalProblem}:
$\argmin\limits_{\bm{\theta} \in \mathbb{R}^p} f(\bm{\theta}),\textup{ s.t. } \|\bm{\theta}\|_1 \leq \lambda,$
where $\lambda$ is a fixed regularization strength \citep{Lasso, park2007l1, hoing_estimation_nodate}. Since there is no explicit expression that can determine $s$ from $\lambda$, we will consider multiple values for $\lambda$ and find out the one that results in a $s$-sparse parameter estimation.  
\item Gradient support pursuit (GraSP), which is implemented in the \texttt{skscope} Python library.
\item Gradient hard thresholding pursuit (GraHTP). We consider two methods for deciding the step size used in GraHTP. One method uses a fixed step size $(2M_{2s})^{-1}$ as suggested by Theorem~5 in \citet{yuan2017gradient}. Since it is difficult to exactly compute $M_{2s}$, we approximate $M_{2s}$ with the largest singular value of $\nabla^2_{\A_0}f(\param_0)$ for instead. Another method considers five step sizes $10^{0}, 10^{-1}, \ldots, 10^{-4}$. For each step size, we run GraHTP and record the objective when it converges. Then we choose the step size that gives the smallest objective. The two methods are denoted as GraHTP1 and GraHTP2, respectively. The two methods are implemented in \texttt{skscope}.
\end{itemize}

We systematically compare these methods in learning sparse linear, sparse classification, and sparse Ising models. The model settings are summarized below.
\begin{itemize}[leftmargin=*]
	\item \textbf{Linear model.} We adopt the settings depicted in the second paragraph on Section~\ref{sec:experiment-exact-solver} except the signal-to-noise ratio is set as $6$.
	\item \textbf{Classification model.} We inherit the settings for $\mathbf{X}$ and $\bm{\theta}^*$ from the linear model in the previous section. The only difference is that the response $y_i$ is sampled from a binomial variable with probability~\eqref{eq:logistic-regression}. 
	\item \textbf{Ising model.} First, we depict the method for generating $\bm{\theta}^* \in \mathbb{R}^{p \times p}$: (i) generate a zero-diagonal upper-triangle matrix $\mathbf{S} \in \mathbb{R}^{p\times p}$ whose $s$ non-zero entries are randomly selected from upper triangle; (ii) for the non-zero entries in $\mathbf{S}$, their values are randomly drawn from $\{-0.5, 0.5\}$ with equal probability; (iii) $\bm{\theta} = \mathbf{S}^\top + \mathbf{S}$. Then, we independently draw $n$ samples from the Ising model in Section~\ref{sec:ising}.
\end{itemize}
For the linear model and classification model, we set $p=500$ and $s=50$; as for the Ising model, we set $p=20$, and $s=40$. We adopt the same evaluation metrics in Section~\ref{sec:experiment-exact-solver}. 
The numerical results of the three models are demonstrated in Figure~\ref {fig:sota}. 

\begin{figure}[t]
	\centering 
    \vspace{-24pt}
	\includegraphics[width=1.0\textwidth]{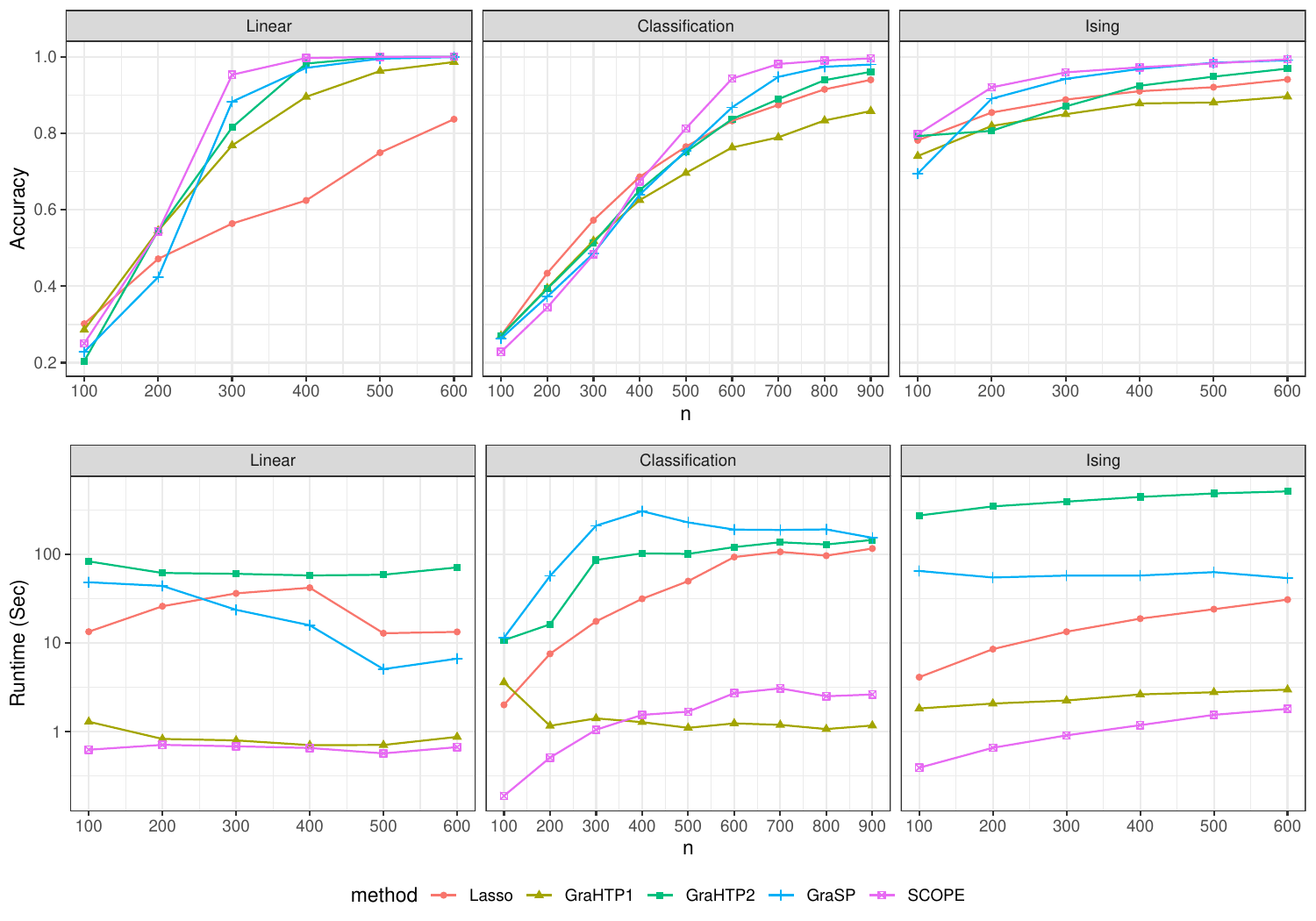}
	\vspace{-24pt}
	\caption{The accuracy (upper panel) and runtime (bottom panel) as the sample size increases from 100 with step size 100. All experiments were independently repeated 100 times. The $y$-axis for runtime is $\log_{10}$-transformed. } 
	\label{fig:sota} 
\end{figure}

Figure~\ref{fig:sota} shows that, among these methods, SCOPE uses the minimal samples to perfectly recover $\mathcal{A}^*$ in the three models. The results also verify the support recovery property of SCOPE in Theorem~\ref{thm:recovery}. Like SCOPE, GraHTP2 can recover $\mathcal{A}^*$ in all of the three models since it has a support recovery guarantee when its hyperparameter, step size, is properly chosen. However, when step size is not chosen appropriately, GraHTP may not correctly recover $\mathcal{A}^*$ when the sample size is large, which can be witnessed by the performance of GraHTP1. In contrast to GraHTP2, GraHTP1 fails to recover $\mathcal{A}^*$ in logistic regression and Ising model when $n \geq 600$ regardless of its successful recovery for $\mathcal{A}^*$ in the linear model. These numerical results suggest tuning the hyperparameter is essential for GraHTP. In terms of GraSP, despite it having no support recovery guarantee, it can successfully identify $\mathcal{A}^*$ under the linear model and Ising regression model, and its accuracy is very close to 1 under logistic regression. Yet, compared with the SCOPE, it generally requires more samples to recover $\mathcal{A}^*$ in these models. Finally, the Lasso-based relaxation method, it can recover $\mathcal{A}^*$ in all models, as the theory suggests. Unfortunately, it requires more samples to reach perfect recovery compared to SCOPE.

From the runtime analysis presented in the bottom panel of Figure~\ref{fig:sota}, SCOPE exhibits competitive runtime compared to existing HT-based methods. Moreover, SCOPE achieves approximately 10--20$\times$ speedup over Lasso-based methods, which are also implemented with a \textsf{C++} backend. This computational efficiency can be attributed to (i) the algorithmic structure of SCOPE and (ii) its efficient implementation in \textsf{C++}. Overall, the numerical results indicate that SCOPE provides a computationally efficient alternative for sparsity-constrained optimization problems.

\section{Conclusion}\label{sec:conclusion}

Motivated by the splicing technique introduced in \citet{abess}, this paper proposes the SCOPE algorithm for solving~\eqref{originalProblem}. Under sufficient conditions, the algorithm recovers the true support set in a linear convergence rate. Impressively, these theoretical guarantees are attained without tuning hyper-parameters. To our knowledge, SCOPE provides a tuning-free alternative among existing HT-type methods. The numerical results show that, in most cases, SCOPE surpasses state-of-the-art solvers in terms of both accuracy and computational efficiency. 

An interesting direction for future research is to integrate the proposed splicing mechanism with IHT-type algorithms or recent Newton hard-thresholding pursuit methods \citep{zhou_global_nodate}. This can handle moderately large values of $s$ and avoid solving subproblems exactly. It is also of interest to incorporate penalty terms into the framework to further expand its applicability.

\acks{We thank the Editor-in-Chief, the Action Editor, and the anonymous reviewers for their valuable comments and suggestions.

Wang’s research is partially supported by National Natural Science Foundation of China grants No.72171216 and 12231017, and the National Key R\&D Program of China No.2022YFA1003803. Lin’s research is partially supported by National Natural Science Foundation of China grants No.12171310 and the Shanghai Natural Science Foundation No.20ZR1421800. The authors declare that they have no competing interests.}

\newpage

\appendix

\section{Proof for Theoretical Results}

\subsection{Proof Sketch for the Core Results}

Theorems~\ref{thm:recovery}, \ref{thm:convergence-rate} and~\ref{thm:support_recovery_relax} are the core theoretical results of the proposed algorithm. We therefore provide a high-level proof sketch here, while deferring the detailed proofs to the subsequent sections.

\begin{proof}{\bf sketch for Theorem~\ref{thm:recovery}.}
	Before proving Theorem~\ref{thm:recovery}, we present three crucial Lemmas. Among them, the first two Lemmas establish the lower and upper bounds for $\nabla^2_{\mathcal{A}}f(\bm{\theta})$ and $\nabla_{\mathcal{A}}\nabla_{\mathcal{B}}f(\bm{\theta})$ where $\bm{\theta}$ is an arbitrary sparse parameter, and $\mathcal{A}, \mathcal{B}$ are arbitrary support sets. 
    \begin{lemma}\label{lemma1}
	If $f:\mathbb{R}^p\mapsto \mathbb{R}$ is twice differentiable and is $m_k$-RSC and $M_k$-RSS, then, for $\forall \mathcal{A} \subseteq [p], \bm{x} \in \mathbb{R}^p$ satisfying $\vert \mathcal{A} \cup \operatorname{supp}(\bm{x})\vert  \leq k$,
	we have:$\quad m_{k} I_{\vert \mathcal{A}\vert } \preceq \nabla_{\mathcal{A}}^{2} f(\bm{x}) \preceq M_{k}I_{\vert \mathcal{A}\vert }.$
    \end{lemma}
    \begin{lemma}\label{lemma2}
    	If $f:\mathbb{R}^p \mapsto \mathbb{R}$ is $m_k$-RSS and $M_k$-RSC and has the second differentiation, then
    	$\forall \mathcal{A},\mathcal{B} \subseteq [p], \bm{x} \in \mathbb{R}^p$ satisfying $\vert \mathcal{A} \cup \mathcal{B} \cup \operatorname{supp}(\bm{x})\vert  \leq k \text{ and } \mathcal{A} \cap \mathcal{B} = \emptyset$, we have: 
    	$$\|\nabla_{\mathcal{A}} \nabla_{\mathcal{B}} f(\bm{x}) v\| \leq \frac{M_{k}-m_{k}}{2} \| v\| ~\textup{ and }~ w^{\top} \nabla_{\mathcal{A}} \nabla_{\mathcal{B}} f(\bm{x}) v  \leq \frac{M_{k}-m_{k}}{2}\|w\| \|v\|,$$
    	$\forall v \in \mathbb{R}^{\vert \mathcal{A}\vert }, w \in \mathbb{R}^{\vert \mathcal{B}\vert }$.
    \end{lemma}
    Together with Taylor expansion at $\bm{\theta}^*$, the two Lemmas facilitate the control for $f(\bm{\theta}) - f(\bm{\theta}^*)$. These two Lemmas can be proven under Assumption~\ref{con:convex-smooth}. 
	
	Next, Lemma~\ref{lemma:loss-now-next} shows that whenever we encounter a sparse solution whose support set is not identical to $\mathcal{A}^*$, then the objective function value at this sparse solution can be bounded. This bound is derived from the previous two Lemmas and Assumption~\ref{con:bound-gradient}. More interestingly, this lemma shows that, for the following sparse solution derived by one splicing operator, the value of the objective function at this new solution can be controlled by similar terms for bounding the previous sparse solution. 
    Therefore, Lemma~\ref{lemma:loss-now-next} enables the comparison of $f(\cdot)$ at the two sparse solutions. 
    \begin{lemma}\label{lemma:loss-now-next}
    	Let $\widehat{\bm{\theta}}$ is any estimation in the splicing loop whose support set $\widehat{\mathcal{A}}$ does not equal to $\mathcal{A}^*$, i.e., $\I_1 \coloneqq \widehat{\I} \cap \mathcal{A}^* \neq \emptyset$, and we define $\widetilde{\bm{\theta}}$ as an estimation under a new support set $\widetilde{\mathcal{A}}$ given by one splicing operation: $\widetilde{\mathcal{A}} \coloneqq (\operatorname{supp}\{\widehat{\bm{\theta}}\} \cup \mathcal{S}^\I_{\vert \I_1\vert }) \setminus \mathcal{S}^\mathcal{A}_{\vert \I_1\vert }$. 
    	Support Assumptions~\ref{con:convex-smooth}-\ref{con:bound-gradient} hold, and let $m \coloneqq m_{3s}$, $M \coloneqq M_{3s}$, $v \coloneqq \frac{M-m}{2}$, $c\coloneqq\frac{0.35}{\sqrt{s}}(1.49m-M)$, then we have: 
    	\begin{equation}\label{gapOfLoss:beforeSplicing}
			f(\widehat{\bm{\theta}})-f (\bm{\theta}^{*} )
			\geq 
			\left(\frac{m}{2} - \sqrt{2s}c\right) \|\param_{\I_{1}}^{*} \|^{2}, 
    	\end{equation}
            \begin{equation}\label{gapOfLoss:afterSplicing}
            \begin{aligned}
                \vert  f(\widetilde{\bm{\theta}})-f (\bm{\theta}^{*} ) \vert 
                \leq \ & \left[\frac{s c^{2}}{m}+\frac{s c^{2} M}{2 m^{2}} \right. 
                         + \left. (\frac{\sqrt{s} c v}{m}
                        +\frac{\sqrt{s} c M v}{m^{2}}+\sqrt{s} c ) \Delta \right.  + \left. (\frac{M v^{2}}{2 m^{2}}+\frac{M}{2} ) \Delta^{2} \right] \|\param_{\I_{1}}^{*} \|^{2},
            \end{aligned}
            \end{equation}

		where $\Delta$ is a constant depends on $m, M$ and $s$.
    \end{lemma}

	We prove Theorem~\ref{thm:recovery} by showing that if $\widehat{\mathcal{A}} \neq \mathcal{A}^*$, it would cause a contradiction. Notably, Lemma~\ref{lemma:loss-now-next} allows us to analyze the loss at a support set $\widetilde{A}$ attained by one splicing operation upon $\widehat{\mathcal{A}}$. By comparing the $f$ under $\widehat{\mathcal{A}}$ or $\widetilde{\mathcal{A}}$, we find 
	$\min\limits_{\operatorname{supp}\{\bm{\theta}\} = \widehat{\mathcal{A}}} f(\bm{\theta})
	> 
	\min\limits_{\operatorname{supp}\{\bm{\theta}\} = \widetilde{\mathcal{A}}} f(\bm{\theta})$
	hold under Assumption~\ref{con:bound-gradient}. 
	This implies $\widehat{\mathcal{A}}$ should not be the support convergence point, which contradicts the fact that Algorithm~\ref{algo:main} outputs $\widehat{\mathcal{A}}$. Therefore, our primary supposition is incorrect, and $\widehat{\mathcal{A}}=\mathcal{A}^*$ must hold.
	
	
\end{proof}

\begin{proof}{\bf sketch of Theorem~\ref{thm:convergence-rate}.}
	From Lemma~\ref{lemma:loss-now-next}, we can give the bounds of objective function values on the current solution $\bm{\theta}^t$ and the next solution $\bm{\theta}^{t+1}$. By the RIP-type condition implied by Assumption~\ref{con:bound-gradient}, we will show that the ratio of these two bounds is not more than a constant that is less than $1$. Therefore, the algorithm converges linearly.
\end{proof}

\begin{proof}{\bf sketch of Theorem~\ref{thm:support_recovery_relax}.}
To prove Theorem~\ref{thm:support_recovery_relax}, we reuse the results of Lemmas~\ref{lemma1} and~\ref{lemma2}, and further, we develop Lemma~\ref{lemma:loss_down_bound} that bounds the gap of objective between the current solution and the spliced set.  
\begin{lemma}\label{lemma:loss_down_bound}
Let $\widehat{\param}$ be the $s$-sparse solution under a support set $\widehat{\A}$. Denote $\rho_0 \coloneqq 0, \rho_j \coloneqq \max \left\{ |\widehat{\param}_i|: i \in \mathcal{S}_{\A}^{(j)}\right\} / \min \left\{| \nabla_i \loss(\widehat{\param})| : i \in \mathcal{S}_{\I}^{(j)}\right\}$ (for $j \in[s^*]$), and $\rho_{s^{*}+1} \coloneqq +\infty$. Under Assumption~\ref{con:convex-smooth-relaxed}, for $\forall \varepsilon>0$, there exists $k \in\{0, \ldots, s^*\}$ satisfying $\frac{1}{M_{s+s^*}+\varepsilon} \in[\rho_k, \rho_{k+1})$ such that, given the support set after splicing $\widetilde{\A}^k \coloneqq (\widehat{\A}\setminus\mathcal{S}_{\A}^{(k)}) \cup \mathcal{S}_{\I}^{(k)}$, we have:
\begin{equation}\label{eqn:loss_down_bound}
    \argmin\limits_{\operatorname{\supp}\{\param\} = \widetilde{\A}^k} f(\param) - f(\widehat{\param}) \leq \frac{-\varepsilon}{2(M_{s+s^*}+\varepsilon)^2}\left\|\nablaf(\widehat{\param})_{	\mathcal{S}_\I^{(k)}}\right\|^2.
\end{equation}
\end{lemma}
Lemma~\ref{lemma:loss_down_bound} shows that, at each splicing iteration, Algorithm~\ref{algo:main} identifies an appropriate splicing set based on the objective value.
\end{proof}

\subsection{Proof under RIP-type Conditions}\label{sec:proof-rip}
    \begin{proof}{\textbf{of Lemma~\ref{lemma:loss-now-next}.}}
    Define $\mathcal{I}_{2}=\hat{\mathcal{I}} \cap \mathcal{I}^{*}, \mathcal{A}_{1}=\hat{\mathcal{A}} \cap \mathcal{A}^{*}, \mathcal{A}_{2}=\hat{\mathcal{A}} \cap \mathcal{I}^{*}.$
	For ease notations, we denote $\boldsymbol{d}^{*}=\nabla \loss (\boldsymbol{\theta}^{*}), \hat{\boldsymbol{d}}=\nabla \loss(\hat{\boldsymbol{\theta}})$, and $\mathcal{S}_{1}= \mathcal{S}_\mathcal{A}^{(\vert \mathcal{I}_1\vert)}$, where $\vert \mathcal{I}_1\vert$ means the cardinality of $\mathcal{I}_1$.
	Also, denote $\mathcal{A}_{11}=\mathcal{A}_{1} \cap (\mathcal{S}_{1} )^{c}, \mathcal{A}_{12}=\mathcal{A}_{1} \cap \mathcal{S}_{1}, \mathcal{A}_{21}=\mathcal{A}_{2} \cap (\mathcal{S}_{1} )^{c}, \mathcal{A}_{22}=\mathcal{A}_{2} \cap \mathcal{S}_{1} .$
	Then $\mathcal{S}_{1}=\mathcal{A}_{12} \cup \mathcal{A}_{22}$. Since $ \vert \mathcal{A}_{12} \vert + \vert \mathcal{A}_{22} \vert = \vert \mathcal{I}_{1} \vert =\vert \mathcal{A}^*\vert -\vert \mathcal{A}_1\vert   \leq \vert \hat{\mathcal{A}}\vert -\vert \mathcal{A}_1\vert =
	\vert \mathcal{A}_{2} \vert = 
	\vert \mathcal{A}_{21} \vert + \vert \mathcal{A}_{22} \vert $, we have $\vert \mathcal{A}_{12}\vert \leq\vert \mathcal{A}_{21}\vert $.
	
	By definition of $\mathcal{S}_{1}$, we have $\hat{\boldsymbol{\theta}}_{j}^{2} \leq \ \hat{\boldsymbol{\theta}}_{i}^{2},\  \forall j \in \mathcal{A}_{12},\ i \in \mathcal{A}_{21}.$
	Thus,
	\begin{equation}\label{S1}
		 \|\hat{\boldsymbol{\theta}}_{\mathcal{A}_{12}} \| \leq \|\hat{\boldsymbol{\theta}}_{\mathcal{A}_{21}} \|.
	\end{equation}
	Since $\hat{\boldsymbol{\theta}}$ minimizes $\loss(\boldsymbol{\theta})$ given active set $\hat{\mathcal{A}}$,
		$0=\nabla_{\hat{\mathcal{A}}} \loss(\hat{\boldsymbol{\theta}})=\nabla_{\hat{\mathcal{A}}} \loss (\boldsymbol{\theta}^{*} )+\nabla_{\hat{\mathcal{A}}}^{2} \loss(\bar{\boldsymbol{\theta}}) (\hat{\boldsymbol{\theta}}_{\hat{\mathcal{A}}}-\param_{\hat{\mathcal{A}}}^{*} )+\nabla_{\hat{\mathcal{A}}} \nabla_{\mathcal{I}_{1}} \loss(\bar{\boldsymbol{\theta}}) (-\param_{\mathcal{I}_{1}}^{*} )$,
	where $\overline{\boldsymbol{\theta}}=t \hat{\boldsymbol{\theta}}+(1-t) \boldsymbol{\theta}^{*}, 0 \leq t \leq 1$. 
	By Lemma~\ref{lemma1} and Lemma~\ref{lemma2}, we have
	\begin{equation}\label{eq:est_true_gap}
		\begin{aligned}
			\|\hat{\boldsymbol{\theta}}_{\hat{\mathcal{A}}}-\param_{\hat{\mathcal{A}}}^{*} \| 
            &\leq\| [\nabla_{\hat{\mathcal{A}}}^{2} \loss(\bar{\boldsymbol{\theta}}) ]^{-1} \boldsymbol{d}_{\hat{\mathcal{A}}}^{*}\|
            +\| [\nabla_{\hat{\mathcal{A}}}^{2} \loss(\bar{\boldsymbol{\theta}}) ]^{-1} \cdot\nabla_{\hat{\mathcal{A}}}\nabla_{\mathcal{I}_{1}} \loss(\bar{\boldsymbol{\theta}}) \cdot \param_{\mathcal{I}_{1}}^{*} \| 
			\leq  m^{-1} \|\boldsymbol{d}_{\hat{\mathcal{A}}}^{*} \|+m^{-1} v \|\param_{\mathcal{I}_{1}}^{*} \|. 
		\end{aligned}
	\end{equation}
	Thus, the loss function at the point $\hat{\boldsymbol{\theta}}$ satisfies:
		\begin{align*}
			\loss(\hat{\param})-\loss (\param^{*} ) 
            &\geq (\hat{\param}-\param^*)^{\top}\boldsymbol{d}^* + \frac{m}{2}\|\hat{\param}-\param^*\|^2
			\geq \frac{m}{2}\|\hat{\param}-\param^*\|^2 - \|\boldsymbol{d}^*_{\mathcal{I}_2^c}\|\|\hat{\param}-\param^*\| \\
			&= \frac{m}{2}(\|\hat{\param}-\param^*\| - \frac{1}{m}\|\boldsymbol{d}^*_{\mathcal{I}_2^c}\|)^2 - \frac{1}{2m}\|\boldsymbol{d}^*_{\mathcal{I}_2^c}\|^2.
		\end{align*}
 
	By $\|\hat{\param}-\param^*\|\geq\|\param^*_{\mathcal{I}_1}\| \geq \vartheta \geq \frac{\sqrt{s}}{0.35(1.49-\frac{M}{m})m}\|\boldsymbol{d}^*\|_{\infty} \geq \frac{\sqrt{2s}}{m}\|\boldsymbol{d}^*\|_{\infty} \geq \frac{1}{m}\|\boldsymbol{d}^*_{\mathcal{I}_2^c}\|$,
    \begin{equation}
        \begin{aligned}
            \loss(\hat{\param})-\loss (\param^{*} ) 
            &\geq \frac{m}{2} \|\param^*_{\mathcal{I}_1}\|^2 - \|\boldsymbol{d}^*_{\mathcal{I}_2^c}\|\|\param^*_{\mathcal{I}_1}\| 
            \geq \left(\frac{m}{2} - \sqrt{2s}c\right)\|\param^*_{\mathcal{I}_1}\|^2.
        \end{aligned}
    \end{equation}

	Next, we derive an upper bound for $\| \boldsymbol{\theta}^*_{\mathcal{A}_{12}} \|$ and  $\| \boldsymbol{\theta}^*_{\mathcal{I}_{12}} \|$ (defined below) to bound the loss function at the point $\tilde{\boldsymbol{\theta}}$.
	Notably, from inequality~\eqref{eq:est_true_gap}, it can be derived that
	\begin{align*}
		\|\hat{\boldsymbol{\theta}}_{\mathcal{A}_{12}} \| &\geq \|\param_{\mathcal{A}_{12}}^{*} \| -\|\hat{\boldsymbol{\theta}}_{\hat{\mathcal{A}}}-\param_{\hat{\mathcal{A}}}^{*} \| \geq \|\param_{\mathcal{A}_{12}}^{*} \|-\frac{1}{m} \|\boldsymbol{d}_{\hat{\mathcal{A}}}^{*} \|-\frac{v}{m} \|\param_{\mathcal{I}_1}^{*} \|, 
        \\
		\|\hat{\boldsymbol{\theta}}_{\mathcal{A}_{21}} \| &\leq \|\boldsymbol{\theta}^*_{\mathcal{A}_{21}} \|+\|\hat{\boldsymbol{\theta}}_{\hat{\mathcal{A}}}-\param_{\hat{\mathcal{A}}}^{*} \|\leq \frac{1}{m} \|\boldsymbol{d}_{\hat{\mathcal{A}}}^{*} \|+\frac{v}{m} \|\param_{\mathcal{I}_{1}}^{*} \|.
	\end{align*}
	Furthermore, by \eqref{S1}, we have
	\begin{align}\label{S3}
		\|\param_{\mathcal{A}_{12}}^{*} \| \leq 2 (\frac{1}{m} \|\boldsymbol{d}_{\hat{\mathcal{A}}}^{*} \|+\frac{v}{m} \|\param_{\mathcal{I}_{1}}^{*} \| ) \notag
            \leq \frac{2}{m}  (c\sqrt{\frac{s}{ \vert \mathcal{I}_{1} \vert }}+v ) \|\param_{\mathcal{I}_{1}}^{*} \|\notag
		\leq 2 (\frac{\sqrt{s} c}{m}+\frac{v}{m} ) \|\param_{\mathcal{I}_{1}}^{*} \|, 
	\end{align}
	where the second inequality follows from Assumption~\ref{con:bound-gradient}.
	
	On the other hand, denote $\mathcal{S}_{2}=\mathcal{S}_\mathcal{I}^{(\vert \mathcal{I}_1\vert)}$ and $\mathcal{I}_{11}=\mathcal{I}_{1} \cap \mathcal{S}_{2}, \mathcal{I}_{12}=\mathcal{I}_{1} \cap (\mathcal{S}_{2} )^{c}, \mathcal{I}_{21}=\mathcal{I}_{2} \cap \mathcal{S}_{2}, \mathcal{I}_{22}=\mathcal{I}_{2} \cap (\mathcal{S}_{2} )^{c}$.
    
    Since $\vert \mathcal{I}_{11} \vert + \vert \mathcal{I}_{21} \vert =\vert \mathcal{S}_{2}\vert = \vert \mathcal{I}_{1} \vert = \vert \mathcal{I}_{11} \vert + \vert \mathcal{I}_{12} \vert ,
	\text{ we have }
	\vert \mathcal{I}_{12} \vert = \vert \mathcal{I}_{21} \vert $.
	By definition of $\mathcal{S}_{2}$, we have $\hat{\boldsymbol{d}}_{i}^{2} \leq  \hat{\boldsymbol{d}}_{j}^{2}, \quad \forall i \in \mathcal{I}_{12}, j \in \mathcal{I}_{21}.$
	Thus,
	\begin{equation}\label{S4}
		\|\hat{\boldsymbol{d}}_{\mathcal{I}_{12}} \| \leq  \|\hat{\boldsymbol{d}}_{\mathcal{I}_{21}} \|.
	\end{equation}
	Note that,
    \begin{equation}\label{eq:grad_on_inactive}
    \begin{aligned}
        \hat{\boldsymbol{d}}_{\hat{\mathcal{I}}} &= \nabla_{\hat{\mathcal{I}}} \loss(\hat{\boldsymbol{\theta}}) = \nabla_{\hat{\mathcal{I}}} \loss (\boldsymbol{\theta}^{*} ) + \nabla_{\hat{\mathcal{A}}} \nabla_{\hat{\mathcal{I}}} \loss(\bar{\boldsymbol{\theta}}) (\hat{\boldsymbol{\theta}}_{\hat{\mathcal{A}}}-\param_{\hat{\mathcal{A}}}^{*} ) + \nabla_{\mathcal{I}_{1}} \nabla_{\hat{\mathcal{I}}} \loss(\bar{\boldsymbol{\theta}}) (-\param_{\mathcal{I}_{1}}^{*} ),
    \end{aligned}
    \end{equation}
	where $\overline{{\boldsymbol{\theta}}}=t \hat{\boldsymbol{\theta}}+(1-t) \boldsymbol{\theta}^{*}, 0 \leq t \leq 1$. Then, by (\ref{eq:grad_on_inactive}),
	\begin{align*}
		\begin{aligned}
			\|\hat{\boldsymbol{d}}_{\mathcal{I}_{12}} \| 
            \geq& - \|\nabla_{\mathcal{I}_{12}}\loss(\boldsymbol{\theta}^{*}) \|- \|\nabla_{\hat{\mathcal{A}}} \nabla_{\mathcal{I}_{12}} \loss(\bar{\boldsymbol{\theta}}) (\hat{\boldsymbol{\theta}}_{\hat{\mathcal{A}}}-\param_{\hat{\mathcal{A}}}^{*} ) \|
            + \|\nabla_{\mathcal{I}_{12}}^{2} \loss(\bar{\boldsymbol{\theta}}) \param_{\mathcal{I}_{12}}^{*} \|- \|\nabla_{\mathcal{I}_{11}} \nabla_{\mathcal{I}_{12}} \loss(\bar{\boldsymbol{\theta}}) \param_{\mathcal{I}_{1}}^{*} \|\\
			\geq&- \|\boldsymbol{d}_{\mathcal{I}_{12}}^{*} \|-v \|\hat{\boldsymbol{\theta}}_{\hat{\mathcal{A}}}-\param_{\hat{\mathcal{A}}}^{*} \|+m \|\param_{\mathcal{I}_{12}}^{*} \|-v \|\param_{\mathcal{I}_{1}}^{*} \|
			\geq m \|\param_{\mathcal{I}_{12}}^{*} \|-v \|\param_{\mathcal{I}_{1}}^{*} \|- \|\boldsymbol{d}_{\mathcal{I}_{12}}^{*} \|-\frac{v}{m} \|\boldsymbol{d}_{\hat{\mathcal{A}}}^{*} \|-\frac{v^{2}}{m} \|\param_{\mathcal{I}_{1}}^{*} \|,
            \\
            \|\hat{\boldsymbol{d}}_{\mathcal{I}_{21}} \| 
            \leq& \|\nabla_{\mathcal{I}_{21}} \loss (\boldsymbol{\theta}^{*} ) \|+ \|\nabla_{\hat{\mathcal{A}}} \nabla_{\mathcal{I}_{21}}\loss(\bar{\boldsymbol{\theta}})(\hat{\boldsymbol{\theta}}_{\hat{\mathcal{A}}}-\param_{\hat{\mathcal{A}}}^{*}) \| +\| \nabla_{\mathcal{I}_{1}} \nabla_{\mathcal{I}_{21}} \loss (\bar{\boldsymbol{\theta}}) \param_{\mathcal{I}_{1}}^{*} \|
			\leq \|\boldsymbol{d}_{\mathcal{I}_{21}}^{*} \|+v \|\hat{\boldsymbol{\theta}}_{\hat{\mathcal{A}}}-\param_{\hat{\mathcal{A}}}^{*} \|+v \|\param_{\mathcal{I}_{1}}^{*} \|\\
			\leq& \|\boldsymbol{d}_{\mathcal{I}_{21}}^{*} \|+\frac{v}{m} \|\boldsymbol{d}_{\hat{\mathcal{A}}}^{*} \|+\frac{v^{2}}{ m} \|\param_{\mathcal{I}_1}^{*} \|+v \|\param_{\mathcal{I}_{1}}^{*} \|.
		\end{aligned}
	\end{align*}
	In conjunction with \eqref{S4}, we have 
	$$ \|\boldsymbol{d}_{\mathcal{I}_{21}}^{*} \|+ \frac{v}{m} \|\boldsymbol{d}_{\hat{\mathcal{A}}}^{*} \|+ (\frac{v^{2}}{m}+v ) \|\param_{\mathcal{I}_{1}}^{*} \| \geq m \|\param_{\mathcal{I}_{12}}^{*} \|-v\|\param_{\mathcal{I}_1}^{*}\|-\|\boldsymbol{d}_{\mathcal{I}_{12}}^{*}\| -\frac{v}{m}\| \boldsymbol{d}_{\hat{\mathcal{A}}}^{*}\|-\frac{v^{2}}{m}\| \param_{\mathcal{I}_{1}}^{*} \|.$$
	Thus, we have
	\begin{equation}\label{S5}
		\begin{aligned}
			 \frac{\|\param_{\mathcal{I}_{12}}^{*} \|}{ \|\param_{\mathcal{I}_{1}}^{*} \|}
			\leq& \frac{1}{m\|\param_{\mathcal{I}_{1}}^{*} \|}\left[ \|\boldsymbol{d}_{\mathcal{I}_{21}}^{*} \|+ \|\boldsymbol{d}_{\mathcal{I}_{12}}^{*} \|+ 2 \frac{v}{m} \|\boldsymbol{d}_{\hat{\mathcal{A}}}^{*} \| + 2 (\frac{v^{2}}{m}+v ) \|\param_{\mathcal{I}_{1}}^{*} \|\right]    \notag\\
			\leq& \frac{c}{ \sqrt{\vert \mathcal{I}_{1} \vert } m} \left( \sqrt{\vert \mathcal{I}_{21} \vert }+ \sqrt{\vert \mathcal{I}_{12}\vert }+ \frac{2v}{m} \sqrt{s} \right)\notag + 2 (\frac{v^{2}}{m^{2}}+\frac{v}{m} )   \notag\\
			\leq& 2 (\frac{c }{m }\sqrt{\frac{\vert \mathcal{I}_{12} \vert }{\vert \mathcal{I}_{1} \vert }}+\frac{\sqrt{s}cv}{m^{2} \sqrt{\vert \mathcal{I}_{1} \vert }}+\frac{v^{2}}{m^{2}}+\frac{v}{m} )  \notag 
			\leq 2 (\frac{\sqrt{s} c}{m}+\frac{\sqrt{s}cv}{m^{2}}+\frac{v^2}{m^{2}}+\frac{v}{m} ).
		\end{aligned}
	\end{equation}
	
	Consider the new active set $\tilde{\mathcal{A}}= (\hat{\mathcal{A}} \backslash \mathcal{S}_{1} ) \cup \mathcal{S}_{2}$ and inactive set $\tilde{\mathcal{I}}=(\tilde{\mathcal{A}})^{c}$. 
	Similar to $\mathcal{I}_{1}$, define $\tilde{\mathcal{I}}_{1}=\tilde{\mathcal{I}} \cap \mathcal{A}^{*}$. 
	Since $\tilde{\boldsymbol{\theta}}=\arg\min\limits_{\param_{\tilde{\mathcal{I}}}=0} \loss(\boldsymbol{\theta})$, similar to \eqref{eq:est_true_gap}, we have
	$$ \|\tilde{\boldsymbol{\theta}}_{\tilde{\mathcal{A}}}-\param_{\tilde{\mathcal{A}}}^{*} \| \leq m^{-1} \|\boldsymbol{d} _{\tilde{\mathcal{A}}}^{*} \|+\frac{v}{m} \|\param_{\tilde{\mathcal{I}}_1}^{*} \|.$$
	Notice that $\tilde{\I_1}=\mathcal{I}_{12} \cup\mathcal{A}_{12}$, the gap between the loss function at point $\tilde{\boldsymbol{\theta}}$ and $\loss(\boldsymbol{\theta})$ can be expanded as:
	\begin{equation}
		\begin{aligned}
			\vert  \loss(\tilde{\boldsymbol{\theta}})-\loss (\boldsymbol{\theta}^{*} ) \vert 
			\leq& \vert  \nabla_{\tilde{\mathcal{A}}} \loss(\boldsymbol{\theta}^{*})^\top (\tilde{\boldsymbol{\theta}}_{\tilde{\mathcal{A}}}-\param_{\tilde{\mathcal{A}}}^{*} )\vert+\vert\nabla_{\tilde{\mathcal{I}}_{1}} \loss(\boldsymbol{\theta}^{*})^{\top} (-\param_{\tilde{\mathcal{I}}_{1}}^{*} )\vert
                +\frac{M}{2} \|\tilde{\param}-\param^*\|^2			\\
			\leq& \|\boldsymbol{d}^{*}_{\tilde{\mathcal{A}}} \| \|\tilde{\boldsymbol{\theta}}_{\tilde{\mathcal{A}}}-\param_{\tilde{\mathcal{A}}}^{*} \|+ \|\boldsymbol{d}_{\tilde{\mathcal{I}}_{1}}^{*} \| \|\param_{\tilde{\mathcal{I}}_1}^{*} \|+\frac{M}{2} \|\tilde{\boldsymbol{\theta}}_{\tilde{\mathcal{A}}}-\param_{\tilde{\mathcal{A}}}^{*} \|^{2} 
                +\frac{M}{2}  \|\param_{\tilde{\mathcal{I}}_1}^{*} \|^{2}\\
			\leq&m^{-1}\|\boldsymbol{d}_{\tilde{\mathcal{A}}}^{*}\|^{2}+\frac{v}{m}\|\boldsymbol{d}_{\mathcal{A}}^{*}\|\|\param_{\mathcal{I}_{12} \cup\mathcal{A}_{12}}^{*}\| 
                +\|\boldsymbol{d}_{\tilde{\mathcal{I}}_{1}}^{*}\|\|\param_{\mathcal{I}_{12} \cup \mathcal{A}_{12}}^{*}\|+\frac{M}{2 m^{2}}\|\boldsymbol{d}_{\tilde{\mathcal{A}}}^{*}\|^{2}\\
			&{+\frac{Mv}{ m^{2}}\|\boldsymbol{d}_{\tilde{\mathcal{A}}}^{*}\|\|\param_{\mathcal{I}_{12} \cup \mathcal{A}_{12}}^{*}\|+ (\frac{M v^{2}}{2 m^{2}}+\frac{M}{2}  ) \|\boldsymbol{\theta}^{*}_{\mathcal{I}_{12} \cup \mathcal{A}_{12}} \|^{2}}\\
			\leq& (\frac{1}{m}+\frac{M}{2 m^{2}} ) \frac{s c^{2}}{ \vert \mathcal{I}_{1} \vert } \|\param_{\mathcal{I}_{1}}^{*} \|^{2} 
                + c\left[ (\frac{v}{m}+\frac{M v}{ m^{2}} ) \sqrt{\frac{s}{ \vert \mathcal{I}_{1} \vert }}+\sqrt{\frac{ \vert \tilde{\mathcal{I}}_{1} \vert }{ \vert \mathcal{I}_{1} \vert }} \right] \frac{\| \param_{\mathcal{I}_{12} \cup \mathcal{A}_{12}}^{*} \|}{ \|\param_{\mathcal{I}_{1}}^{*} \|} \|\boldsymbol{\theta}^*_{\mathcal{I}_{1}} \|^2 \\
			&+ (\frac{M v^{2}}{2 m^{2}}+\frac{ M}{2} ) \frac{\| \param_{\mathcal{I}_{12} \cup \mathcal{A}_{12}}^{*} \|^2}{ \|\param_{\mathcal{I}_{1}}^{*} \|^{2}} \|\param_{\mathcal{I}_{1}}^{*} \|^{2}\\
			\leq& \left[\frac{s c^{2}}{m}+\frac{s c^{2} M}{2 m^{2}}+ (\frac{\sqrt{s} c v}{m}+\frac{\sqrt{s} c M v}{ m^{2}}+\sqrt{s} c ) \Delta \right.
                + \left.(\frac{M v^{2}}{2 m^{2}}+\frac{M}{2} ) \Delta^{2} \right] \|\param_{\mathcal{I}_{1}}^{*} \|^{2},
		\end{aligned}
	\end{equation}
	where $ \Delta=2 \sqrt{ (\frac{\sqrt{s} c}{m}+\frac{v}{m} )^{2}+ (\frac{\sqrt{s} c}{m}+\frac{\sqrt{s} c v}{m^{2}}+\frac{v^{2}}{m^{2}}+\frac{v}{m} )^{2}}.$
	This encloses the proof of Lemma~\ref{lemma:loss-now-next}.
\end{proof}
\begin{proof}{\textbf{of Theorem~\ref{thm:recovery}.}}
	Suppose $\hat{\mathcal{A}} \neq \mathcal{A}^*$, then according to inequalities~\eqref{gapOfLoss:beforeSplicing} and~\eqref{gapOfLoss:afterSplicing} in Lemma~\ref{lemma:loss-now-next}, we have:
	\begin{align*}
		\begin{aligned}
		\frac{\loss(\hat{\boldsymbol{\theta}})-\loss(\tilde{\boldsymbol{\theta}})}{m \|\param_{\mathcal{I}_{1}}^{*} \|^{2}}
			\geq \frac{1}{2} -\sqrt{2}\frac{\sqrt{s}c}{m}-(\frac{\sqrt{s}c}{m} )^{2}  - \frac{1}{2}(\frac{\sqrt{s}c}{m} )^{2} \frac{M}{m}
			-\frac{\sqrt{s}c}{m} ( \frac{v}{m}+\frac{M}{m} \frac{v}{m}+1 ) \Delta- (\frac{1}{2} \frac{M}{m} (\frac{v}{m} )^{2}+\frac{1}{2} \frac{M}{m} ) \Delta^{2}.
		\end{aligned}
	\end{align*}
	
	Denote $x=\frac{M}{m}$, then $\frac{v}{m}=\frac{x-1}{2}$, and $\frac{\sqrt{s}c}{m}=0.35(1.49-x)$, we have
	$$\loss(\hat{\boldsymbol{\theta}})- \loss(\tilde{\boldsymbol{\theta}}) \geq m \|\param_{\mathcal{I}_{1}}^{*} \|^{2} \varepsilon(\frac{M}{m}),$$ 
	where $\varepsilon(x)=
		\frac{1}{2}-\sqrt{2}(0.5215-0.35x)-(0.5215-0.35x)^2-\frac{1}{2}x(0.5215-0.35x)^2 
		- \frac{1}{2}(0.5215-0.35x)(x^2+1)\Delta(x) - \frac{1}{8}x(x^2-2x+5)\Delta(x)^2$,
	and $\Delta(x)=(0.15x+0.0215) \sqrt{x^{2}+2 x+5}$.
	
	By Assumption~\ref{con:bound-gradient}, we know $x\in [1,1.49]$, and in this interval, it's easy to prove $\varepsilon(x) > 0.0125$ which leads to a contradiction, thus $\hat{\mathcal{A}} = \mathcal{A}^{*}.$
\end{proof}
\begin{proof}{\textbf{of Theorem~\ref{thm:convergence-rate}.}}
	We reuse the notions such as $m, M, v, x, \Delta(x)$ in the proof of Theorem~\ref{thm:recovery}. Let $\Delta=2 \sqrt{ (\frac{\sqrt{s} c}{m}+\frac{v}{m} )^{2}+ (\frac{\sqrt{s} c}{m}+\frac{\sqrt{s} c v}{m^{2}}+\frac{v^{2}}{m^{2}}+\frac{v}{m} )^{2}}$, then by Assumption~\ref{con:bound-gradient}, we have:
	\begin{equation}\label{S10}
		\begin{aligned}
			\frac{ \vert  \loss(\boldsymbol{\theta}^{t+1})-\loss(\boldsymbol{\theta}^{*}) \vert  }{ \vert  \loss(\boldsymbol{\theta}^{t})-\loss(\boldsymbol{\theta}^{*}) \vert  } 
			\leq
			& \frac{\frac{s c^{2}}{m}+\frac{s c^{2} M}{2 m^{2}}+ (\frac{\sqrt{s} c v}{m}+\frac{\sqrt{s} c M v}{ m^{2}}+\sqrt{s} c ) \Delta+ (\frac{M v^{2}}{2 m^{2}}+\frac{M}{2} ) \Delta^{2} }{\frac{m}{2} - \sqrt{2s}c}\\
			\leq
			&[\frac{1}{2}-\sqrt{2}(0.5215-0.35x)]^{-1}[(0.5215-0.35x)^2\\
                &+\frac{1}{2}x(0.5215-0.35x)^2 
				+ \frac{1}{2}(0.5215-0.35x)(x^2+1)\Delta(x)
                + \frac{1}{8}x(x^2-2x+5)\Delta(x)^2],
		\end{aligned}
	\end{equation}
	where the first inequality follows from  inequalities~\eqref{gapOfLoss:beforeSplicing} and~\eqref{gapOfLoss:afterSplicing} given by Lemma~\ref{lemma:loss-now-next}, and the second inequality comes from $x=\frac{M}{m}$, $\frac{v}{m}=\frac{x-1}{2}$ and $\frac{\sqrt{s}c}{m}=0.35(1.49-x)$. 
	
	Furthermore, from Assumption~\ref{con:bound-gradient}, we know $x \in [1,1.49]$. It is easy to prove that the univariate function in the right-hand side of { \eqref{S10}}, denoted by $\delta(x)$, satisfies $\delta_1^{-1} \coloneqq \max\limits_{x \in [1,1.49]}\delta(x)<0.968$. Consequently, we have:
	$$
	\vert  \loss(\boldsymbol{\theta}^{t+1})-\loss(\boldsymbol{\theta}^{*}) \vert  \leq \delta_1^{-1} \vert  \loss(\boldsymbol{\theta}^{t})-\loss(\boldsymbol{\theta}^{*}) \vert ,
	$$
	which completes the proof.
\end{proof}

 \begin{proof}{\textbf{of Theorem~\ref{thm:convergence}.}}
 	Here, the notations $m, M, v, x, \Delta(x)$ adopt from the proof of Theorem~\ref{thm:recovery}. For each $\boldsymbol{\theta}^t$ whose support set is not equal to $\mathcal{A}^*$, we have
 	\begin{align*}
 		\loss(\boldsymbol{\theta}^t)-\loss(\boldsymbol{\theta}^{*}) 
 		\geq (\frac{1}{2}-\sqrt{2}(0.5215-0.35x))m\vartheta 
 		> 0.2574m\vartheta^2,
 	\end{align*}
 	where the first inequality comes from the inequality~\eqref{gapOfLoss:beforeSplicing} in Lemma~\ref{lemma:loss-now-next}.
 	So, there is a constant $\delta_2>0$ such that
 	\begin{equation} \label{eq:S11}
 		\loss(\boldsymbol{\theta}^t)-\loss\left(\boldsymbol{\theta}^{*}\right) > \delta_2m\vartheta^2.
 	\end{equation}
 	On the other hand, if $t>\log_{\delta_{1}}\left[\frac{\mid\loss(\boldsymbol{\theta}^{0})-\loss(\boldsymbol{\theta}^{*}) \mid}{\delta_{2} m \vartheta^{2}}\right]$, Theorem~\ref{thm:convergence-rate} asserts that 
 \begin{equation} 
 		\vert  \loss(\boldsymbol{\theta}^{t})-\loss(\boldsymbol{\theta}^{*}) \vert  \leq \delta_{1}^{-t} \vert  \loss(\boldsymbol{\theta}^{0})-\loss(\boldsymbol{\theta}^{*}) \vert <\delta_{2} m \vartheta^{2}, 
   \label{eq:S12}
 \end{equation}
 	where the first inequality raises from \eqref{eq:convergence-rate}. 
  However, { \eqref{eq:S11}} contradicts to { \eqref{eq:S12}}, implying $\mathcal{A}^{t}=\mathcal{A}^{*}$ must hold when $t>\log_{\delta_{1}}\left[\frac{\mid\loss(\boldsymbol{\theta}^{0})-\loss(\boldsymbol{\theta}^{*}) \mid}{\delta_{2} m \vartheta^{2}}\right]$.
 \end{proof}

\subsection{Proof without RIP-type Conditions}\label{sec:proof-rip-free}
\begin{proof}{\textbf{of Lemma~\ref{lemma:loss_down_bound}.}}
	We reuse the notions in the proof of Lemma~\ref{lemma:loss-now-next} (e.g., we denote $M=M_{s+s^*}$). It is directly to see that $\rho_0 \leq \rho_1 \leq \cdots \leq \rho_{s^*}\leq\rho_{s^*+1}$. When $M^{-1} \leq \rho_1$, selecting $k = 0$ makes the conclusion trivially hold. Consider the case that $M^{-1} \geq \rho_1$, we construct an auxiliary $s$-sparse vector $\bar{\param} \in \mathbb{R}^p$:
	\begin{align*}
		\bar{\param} = 
		\begin{cases}
			\hat{\param}_i, & \textup{ if } i \in \hat{\mathcal{A}} \backslash \mathcal{S}_\mathcal{A}^{(k)} \\
			\frac{-1}{M+\varepsilon} \hat{\boldsymbol{d}}_i, & \textup{ if } i \in \mathcal{S}_\mathcal{I}^{(k)} \\
			0, & \textup{ if } i \in \mathcal{S}_\mathcal{A}^{(k)} \cup (\hat{\mathcal{I}} \backslash \mathcal{S}_\mathcal{I}^{(k)})
		\end{cases},
	\end{align*}
	from which we can easily notice that
	\begin{equation}\label{eqn:auxiliary_variable_property}
		(M+\epsilon)(\bar{\param} - \hat{\param}) = 
		\begin{cases}
			0, & \textup{ if } i \in \hat{\mathcal{A}}\backslash \mathcal{S}_\mathcal{A}^{(k)} \\
			-\hat{\boldsymbol{d}}_i, & \textup{ if } i \in \mathcal{S}_\mathcal{I}^{(k)} \\
			-(M+\epsilon)\hat{\param}_i, & \textup{ if } i \in \mathcal{S}_\mathcal{A}^{(k)} \\
			0, & \textup{ if } i \in \hat{\mathcal{I}} \backslash \mathcal{S}_\mathcal{I}^{(k)}
		\end{cases}.
	\end{equation}
	
	
	
	
	Then, we can find that
	\begin{align*}
		\|(M+\varepsilon)(\bar{\param}-\hat{\param})+\hat{\boldsymbol{d}}\|^2-\| \hat{\boldsymbol{d}}\|^2
        = \|(M+\varepsilon) \hat{\param}_{\mathcal{S}_\mathcal{A}^{(k)}}\|^2-\| \hat{\boldsymbol{d}}_{\mathcal{S}_\mathcal{I}^{(k)}}\|^2 
		\leq \| \rho_k^{-1} \hat{\param}_{\mathcal{S}_\mathcal{A}^{(k)}}\|^2-\| \hat{\boldsymbol{d}}_{\mathcal{S}_\mathcal{I}^{(k)}} \|^2  \leq 0
	\end{align*}
	where the first inequality holds because of $(M+\epsilon)^{-1} \in [\rho_k, \rho_{k+1})$ and the second inequality holds due to the definition of $\rho_k$.
	By decomposing the left-hand side of the above inequality and simple algebra, we have
	\begin{equation}\label{eqn:auxiliary_variable_error}
		(\bar{\param}-\hat{\param})^T \hat{\boldsymbol{d}} \leq -\frac{M+\varepsilon}{2}\|\bar{\param}-\hat{\param}\|^2.
	\end{equation}
	Next, we can establish an upper bound for  $f(\bar{\param})-f(\hat{\param})$:
	\begin{equation}\label{eqn:auxiliary_variable_objective_bound}
		\begin{aligned}
		f(\bar{\param})-f(\hat{\param}) 
		& \leq(\bar{\param}-\hat{\param})^T \hat{\boldsymbol{d}}+\frac{M}{2}\|\bar{\param}-\hat{\param}\|^2 \\
		&  \overset{\eqref{eqn:auxiliary_variable_error}}{\leq} \frac{-\varepsilon}{2}\|\bar{\param}-\hat{\param}\|^2
		\overset{\eqref{eqn:auxiliary_variable_property}}{=}\frac{-\varepsilon\|\hat{\boldsymbol{d}}_{\mathcal{S}_\mathcal{I}^{(k)}}\|^2}{2(M+\varepsilon)^2}-\frac{\varepsilon\|\hat{\param}_{\mathcal{S}_\mathcal{A}^{(k)}}\|^2}{2} 
		\leq \frac{-\varepsilon \|\hat{\boldsymbol{d}}_{\mathcal{S}_\mathcal{I}^{(k)}}\|^2}{2(M+\varepsilon)^2}.
		\end{aligned}
	\end{equation}
	where the first inequality comes from the fact that $f$ is $M$-RSS. 
	Finally, since both $\tilde{\param}^{(k)}$ and $\bar{\param}$ have the support set $(\hat{\mathcal{A}} \backslash \mathcal{S}_\mathcal{A}^{(k)}) \cup \mathcal{S}_\mathcal{I}^{(k)}$, we have $\loss(\tilde{\param}^{(k)})-f(\hat{\param}) \leq f(\bar{\param})-f(\hat{\param})$ owing to the definition of $\tilde{\param}^{(k)}$. Coupled with~\eqref{eqn:auxiliary_variable_objective_bound}, it leads to the conclusion of the Lemma~\ref{lemma:loss_down_bound}. 
\end{proof}

\begin{proof}{\textbf{of Theorem~\ref{thm:support_recovery_relax}.}}
	We reuse the notions in the proof of Lemma~\ref{lemma:loss-now-next}.
	We first show that the algorithm will not terminate if $\mathcal{A}^*\nsubseteq  \hat{\mathcal{A}} $, i.e., $\mathcal{I}_1=\mathcal{A}^* \backslash \hat{\mathcal{A}} \neq \emptyset$. 
	
	When $\I_1 \neq \emptyset$, we have:
	\begin{equation}\label{eqn:thm_support_recovery_relax_1}
		\|\hat{\param}-\param^*\| \geq \frac{1}{\sqrt{2}}(\|\param_{\mathcal{I}_1}^*\|+\|\hat{\param}_{\mathcal{A}_2}\|) \geq \frac{1}{\sqrt{2}}(\vartheta+\|\hat{\param}_{\mathcal{A}_2}\|).
	\end{equation}
	Then, we can establish an upper bound for $\|\hat{\param}_{\mathcal{A}_2}\|$ with $\frac{\sqrt{2}}{m}\|\hat{\boldsymbol{d}}_{\mathcal{I}_2^C}\|$: 
	\begin{align*}
		\|\hat{\boldsymbol{d}}_{\mathcal{I}_2^C}\| 
		& =\|\boldsymbol{d}_{\mathcal{I}_2^C}^*+\nabla^2_{\mathcal{I}_2^C}\loss(\bar{\param})(\hat{\param}_{\mathcal{I}_2^C}-\param^*_{\mathcal{I}_2^C})\| 
        \geq m\|\hat{\param}-\param^*\|-\|\boldsymbol{d}_{\mathcal{I}_2^C}^*\| \\
		& \overset{\eqref{eqn:thm_support_recovery_relax_1}}{\geq} \frac{m}{\sqrt{2}}\|\hat{\param}_{\mathcal{A}_2}\| + \frac{m}{\sqrt{2}} \vartheta -\sqrt{s+s^*}\|\boldsymbol{d}^*\|_{\infty} 
		\geq \frac{m}{\sqrt{2}}\|\hat{\param}_{\mathcal{A}_2}\|,
	\end{align*}
 where $\overline{\boldsymbol{\theta}}=t \hat{\boldsymbol{\theta}}+(1-t) \boldsymbol{\theta}^{*}, 0 \leq t \leq 1$,  the last inequality comes from Assumption~\ref{con:bound-gradient-relaxed}. 
	Upon this inequality, we can show that 
	\begin{equation}\label{eqn:max_gradient_min_param}
        \begin{aligned}
		\max \left\{\vert \hat{\boldsymbol{d}}_i\vert: i \in \hat{\mathcal{I}}\right\} 
		\geq \frac{1}{\sqrt{s^*}} \|\hat{\boldsymbol{d}}_{\mathcal{I}_1}\| 
		= \frac{1}{\sqrt{s^*}} \|\hat{\boldsymbol{d}}_{\mathcal{I}_2^C}\| 
		\geq \frac{m}{\sqrt{2 s^*}}\|\hat{\param}_{\mathcal{A}_2}\| 
		\geq m \sqrt{\frac{s-s^*}{2 s^*}} \mathop{\min} \{\vert\hat{\param}_i\vert: i \in \hat{\A}\}            
        \end{aligned}
	\end{equation}
	where the equality comes from the fact that 
	$\hat{\boldsymbol{d}}_{\hat\A} = \mathbf{0}$ and the right-most inequality results from $|\A_2| = |\hat{\A} \cap \I^*| = s - |\hat{\A} \cap \A^*| \geq s - s^*$.
	
	According to~\eqref{eqn:max_gradient_min_param}, we have 
	\begin{align*}
		\frac{1}{m}\sqrt{\frac{2s^*}{s-s^*}}
		  \geq \frac{\min \left\{\vert\hat{\param}_i\vert: i \in \hat{\mathcal{A}}\right\}}{\max \left\{\vert \hat{\boldsymbol{d}}_i\vert: i \in \hat{\mathcal{I}}\right\}}  
		= \frac{\max \left\{\vert\hat{\param}_i\vert: i \in \mathcal{S}_\mathcal{A}^{(1)}\right\}}{\min \left\{\vert \hat{\boldsymbol{d}}_i\vert: i \in \mathcal{S}_\mathcal{I}^{(1)}\right\}} = \rho_1,
	\end{align*}
	further, since $s>(1+2\frac{M^2}{m^2})s^*$, we have $\frac{1}{M} > \frac{1}{m}\sqrt{\frac{2s^*}{s-s^*}}$, and thus, $\frac{1}{M} > \rho_1$. 
	Therefore, for $\forall \varepsilon>0$, there exists  $k\geq 1$ satisfying $\frac{1}{M+\varepsilon} \in[\rho_k, \rho_{k+1})$.
	By Lemma~\ref{lemma:loss_down_bound}, for $\tilde{\param}^{(k)}$, it enjoys: 
	$$\loss(\tilde{\param}^{(k)}) - \loss(\hat{\param}) \leq \frac{-\varepsilon}{2(M_{s+s^*}+\varepsilon)^2}\|\nabla\loss(\hat{\param})_{\mathcal{S}_\mathcal{I}^{(k)}}\|^2 < 0,$$ 
	which implies the algorithm shall not terminate. 
	
	Secondly, we will show that $\operatorname{supp}(\param^*)=\operatorname{supp}(\mathcal{H}_{s^*}(\hat{\param}))$ holds. To this end, it is equivalent to prove $\vartheta > \|\param^*-\mathcal{H}_{s^*}(\hat{\param})\|$, and we will prove this via constructing a contradiction. 

	Presume $\vartheta \leq \|\param^*-\mathcal{H}_{s^*}(\hat{\param})\|$, then according to Theorem~1 and Remark 2 of \citet{shen_tight_nodate}, $\mathcal{H}_{s^*}(\hat{\param})$ satisfies a tight bound $\|\param^*-\mathcal{H}_{s^*}(\hat{\param})\| \leq \frac{\sqrt{5}+1}{2}\|\param^*-\hat{\param}\|,$
	and hence, we can conclude that
	\begin{equation}\label{eqn:lower_bound_hat_param_relaxed}
		\vartheta^{-1}\|\param^*-\hat{\param}\| \geq \frac{2}{\sqrt{5}+1}.
	\end{equation} 
	On the other hand, we have 
	\begin{align*}
		\mathbf{0}
		=\nabla_{\hat{\mathcal{A}}} \loss(\hat{\boldsymbol{\theta}})
		=\nabla_{\hat{\mathcal{A}}} \loss (\boldsymbol{\theta}^{*} )+\nabla_{\hat{\mathcal{A}}}^{2} \loss(\bar{\boldsymbol{\theta}}) (\hat{\boldsymbol{\theta}}_{\hat{\mathcal{A}}}-\param_{\hat{\mathcal{A}}}^{*} )
            +\nabla_{\hat{\mathcal{A}}} \nabla_{\mathcal{I}_{1}} \loss(\bar{\boldsymbol{\theta}}) (-\param_{\mathcal{I}_{1}}^{*} ) 
		=\nabla_{\hat{\mathcal{A}}} \loss (\boldsymbol{\theta}^{*} )+\nabla_{\hat{\mathcal{A}}}^{2} \loss(\bar{\boldsymbol{\theta}}) (\hat{\boldsymbol{\theta}}_{\hat{\mathcal{A}}}-\param_{\hat{\mathcal{A}}}^{*} )
	\end{align*}
	where $\bar{\param}=\lambda\param^*+(1-\lambda)\hat{\param}$ for $\lambda \in (0,1)$, and the last inequality holds due to $\hat{\mathcal{A}} \supseteq \mathcal{A}^*$. Hence, we have
	\begin{equation}\label{eqn:upper_bound_hat_param_relaxed}
        \begin{aligned}
		\|\param^*_{\hat{\mathcal{A}}}-\hat{\param}_{\hat{\mathcal{A}}}\|
         = \|[\nabla_{\hat{\mathcal{A}}}^2\loss(\bar{\param})]^{-1}\boldsymbol{d}^*_{\hat{\mathcal{A}}}\| 
        \leq \frac{1}{m}\|\boldsymbol{d}^*_{\hat{\mathcal{A}}}\| \leq \frac{\sqrt{s}}{m}\|\boldsymbol{d}^*\|_{\infty},            
        \end{aligned}
	\end{equation}
	Furthermore, owing to Assumption~\ref{con:bound-gradient-relaxed} and $\|\param^*-\hat{\param}\| = \|\param^*_{\hat{\mathcal{A}}}-\hat{\param}_{\hat{\mathcal{A}}}\|$, we can derive from \eqref{eqn:lower_bound_hat_param_relaxed} and \eqref{eqn:upper_bound_hat_param_relaxed} that
	$$\frac{m}{2\sqrt{s+s^*}} > \vartheta^{-1} \|\boldsymbol{d}^*\|_{\infty} \geq \frac{m}{\sqrt{s}} \vartheta^{-1} \|\param^*-\hat{\param}\| \geq \frac{2m}{(\sqrt{5}+1)\sqrt{s}}.$$ 
	However, we witness a contradiction between the left-most and right-most parts of this inequality.
	Therefore, $\vartheta > \|\param^*-\mathcal{H}_{s^*}(\hat{\param})\|$ holds, and we can claim $\operatorname{supp}(\param^*)=\operatorname{supp}(\mathcal{H}_{s^*}(\hat{\param}))$.
\end{proof}

\begin{proof}{\textbf{of Theorem~\ref{thm:convergence_rate_relax}.}}
	Again, we use the notions in the proof of Theorem~\ref{thm:recovery}. Beside, we denote $\hat{\param} \coloneqq \param^{t}$ and $\tilde{\param} \coloneqq \param^{t+1}$. Let $\delta_4=\frac{m_{s+s^*}}{4 M_{s+s^*}}\left(1-\frac{4 s^*}{s-2 s^*} \frac{M_{s+s^*}^2}{m_{s+s^*}^2}\right), \varepsilon=\left(\frac{1}{2} \delta_4^{-1} \frac{m}{M}-1\right)^{-1} M$, Lemma~\ref{lemma:loss_down_bound} ensures there is $k \in \{0, 1, \ldots, s^*\}$ satisfying $\frac{1}{M+\varepsilon} \in[\rho_k, \rho_{k+1})$ such that $\loss(\tilde{\param}^{(k)}) - \loss(\hat{\param}) \leq \frac{-\varepsilon}{2(M+\varepsilon)^2}\|\hat{\boldsymbol{d}}_{	\mathcal{S}_\mathcal{I}^{(k)}}\|^2=-\frac{\delta_4}{m}(1-2\delta_4\frac{M}{m})\|\hat{\boldsymbol{d}}_{	\mathcal{S}_\mathcal{I}^{(k)}}\|^2.$
	
	We first consider the case that $k=s^*$. In this case, $|\mathcal{S}_\mathcal{I}^{(k)}| \geq |\mathcal{I}_1|$ holds, and thus $\|\hat{\boldsymbol{d}}_{\mathcal{S}_\mathcal{I}^{(k)}}\| \geq \|\hat{\boldsymbol{d}}_{\mathcal{I}_1}\|$ due to the definition of $\mathcal{S}_\mathcal{I}^{(k)}$. Therefore, we can show that 
	\begin{equation}\label{eqn:objective_gap_hat_true}
		\begin{aligned}
			\loss(\hat{\param})-\loss(\param^*)
			\leq& \hat{\boldsymbol{d}}^{\top}(\hat{\param}-\param^*)-\frac{m}{2}\|\hat{\param}-\param^*\|^2 \\
			=& \hat{\boldsymbol{d}}_{\mathcal{I}_2^C}^{\top}(\hat{\param}-\param^*)_{\mathcal{I}_2^C}-\frac{m}{2}\|\hat{\param}_{\mathcal{I}_2^C}-\param_{\mathcal{I}_2^C}^*\|^2 
			= \frac{1}{2 m}\|\hat{\boldsymbol{d}}_{\mathcal{I}_2^C}\|^2-\frac{m}{2}\|\hat{\param}_{\mathcal{I}_2^C}-\param_{\mathcal{I}_2^C}^*-\frac{1}{m}\hat{\boldsymbol{d}}_{\mathcal{I}_2^C}\|^2 \\
			\leq&  \frac{\|\hat{\boldsymbol{d}}_{\mathcal{I}^C_2}\|}{2 m} = \frac{\|\hat{\boldsymbol{d}}_{\mathcal{I}_1}\|}{2 m}
			\leq  \frac{\|\hat{\boldsymbol{d}}_{\mathcal{S}_\mathcal{I}^{(k)}} \|}{2 m} \leq \frac{1}{2\delta_4(1-2\delta_4\frac{M}{m})}(\loss(\hat{\param})-\loss(\tilde{\param}^{(k)})).
		\end{aligned}
	\end{equation}
	It leads to: 
	\begin{align*}
		\loss(\tilde{\param}^{(k)})-\loss(\param^*)
		=& \loss(\tilde{\param}^{(k)}) - \loss(\hat\param) + \loss(\hat\param) -\loss(\param^*)
		\leq [1-2\delta_4(1-2\delta_4\frac{M}{m})](\loss(\hat{\param})-\loss\left(\param^*\right))\\
		\leq& \left(1-\delta_4\right)(f(\hat{\param})-f\left(\param^*\right)),
	\end{align*}
	where the last inequality holds due to $\delta_4 < \frac{m}{4M}$ from the definition of $\delta_4$. With the above inequality and the fact that $f(\tilde{\param}) \leq f(\tilde{\param}^{(k)})$, we complete the proof for the first case.

	Secondly, we turn to the case that $k+1 \leq s^*$. Then, we give a lower bound for $\|\hat{\param}-\param^*\|$. It is easily notice that, $\|\hat{\param}-\param^*\| \geq \|\hat{\param}_{\mathcal{A}_{21}}\|$; and for $\|\hat{\param}_{\mathcal{A}_{21}}\|$, it satisfies
	\begin{align*}
		\|\hat{\param}_{\mathcal{A}_{21}}\| 
		&\geq \sqrt{|\mathcal{A}_{21}|} \max\{|\hat{\param}_i| : i \in \mathcal{S}^{(k)}_{\A}\} 
		\geq \sqrt{s - 2s^*} \max\{|\hat{\param}_i| : i \in \mathcal{S}^{(k)}_{\A}\} 
		= \sqrt{s - 2s^*} \rho_{k+1}\min\{|\hat{\boldsymbol{d}}_i| : i \in \mathcal{S}^{(k)}_{\I}\} \\
		&\geq \sqrt{s - 2s^*} \rho_{k+1} |\mathcal{I}_{12}|^{-\frac{1}{2}} \|\hat{\boldsymbol{d}}_{\mathcal{I}_{12}}\| 
		\geq \sqrt{\frac{s-2 s^*}{s^*}}  \rho_{k+1}\|\hat{\boldsymbol{d}}_{\mathcal{I}_{12}}\| 
	    \geq \frac{1}{M+\varepsilon} \sqrt{\frac{s-2 s^*}{s^*}}\|\hat{\boldsymbol{d}}_{\mathcal{I}_{12}}\|,
	\end{align*}
	where the first and second inequalities result from the definition of $\mathcal{S}^{(k)}_{\A}$ and $|\mathcal{A}_{21}| \geq s-2s^*$, and the third and fourth equalities comes from the definition of $\mathcal{S}^{(k)}_{\I}$ and $|\mathcal{I}_{12}| \leq s^*$. Consequently, we have:
	\begin{equation}\label{eqn:hat_true_bound2}
		\frac{1}{M+\varepsilon} \sqrt{\frac{s-2 s^*}{s^*}}\|\hat{\boldsymbol{d}}_{\mathcal{I}_{12}}\| \leq \|\hat{\param}-\param^*\|
	\end{equation}

	Let $l=1-2\delta_4\frac{M}{m}$ be a root of 
	\begin{equation}\label{eqn:l_equation}
		\frac{m}{4 l}-\frac{m}{2}+\frac{l}{m}(M+\varepsilon)^2 \frac{s^*}{s-2 s^*}=0.
	\end{equation}
	 Then, we have
	\begin{align*}
	\loss(\hat{\param})-\loss(\param^*) 
	 \leq& \hat{\boldsymbol{d}}_{\mathcal{I}_2^C}(\hat{\param}-\param^*)_{\mathcal{I}_2^C}-\frac{m}{2}\|\hat{\param}_{\mathcal{I}_2^C}-\param_{\mathcal{I}_2^C}^*\|\\
	 =& (\frac{m}{4 l}-\frac{m}{2}) \|\hat{\param}-\hat{\param}^*\|^2+\frac{l}{m}\| \hat{\boldsymbol{d}}_{\mathcal{I}_1}\|^2 
        -\| \sqrt{\frac{l}{m}} \hat{\boldsymbol{d}}_{\mathcal{I}_2^C}-\sqrt{\frac{m}{4l}}(\hat{\param}-\param^*)_{\mathcal{I}_2^C} \|^2 \\
	 \leq &\left(\frac{m}{4 l}-\frac{m}{2}\right)\|\hat{\param}-\param^*\|^2+\frac{l}{m}\|\hat{\boldsymbol{d}}_{\mathcal{I}_{12}}\|^2+\frac{l}{m}\|\hat{\boldsymbol{d}}_{\I_{11}}\|^2 \\
	 \leq &\frac{l}{m}\|\hat{\boldsymbol{d}}_{\mathcal{S}_\mathcal{I}^{(k)}}\|^2 + \left(\frac{m}{4 l}-\frac{m}{2}\right)\|\hat{\param}-\param^*\|^2 
     +\frac{l}{m}(M+\varepsilon)^2 \frac{s^*}{s-2s^*}\|\hat{\param}-\param^*\|^2\\
	 =&\frac{l}{m}\|\hat{\boldsymbol{d}}_{\mathcal{S}_\mathcal{I}^{(k)}}\|^2 
	\leq \frac{l}{\delta_4(1-2\delta_4\frac{M}{m})}(\loss(\hat{\param})-\loss(\tilde{\param}^{(k)})) 
	 = \delta_4^{-1}(\loss(\hat{\param})-\loss(\tilde{\param}^{(k)})),
	\end{align*}
	where the second inequality holds owing to $\I_1 = \I_{12} \cup \I_{11}$, the third inequality holds because of \eqref{eqn:hat_true_bound2} and the definition of $\mathcal{S}_{\I}^{(k)}$, the second equation holds because of~\eqref{eqn:l_equation}. Upon this inequality, simple algebra shows that $\loss(\tilde{\param}^{(k)})-\loss\left(\param^*\right) \leq \left(1-\delta_4\right)(\loss(\hat{\param})-\loss(\param^*))$. 
	According to the definition of $\tilde{\param}$, 
	we can easily derive that $\loss(\tilde{\param})-\loss\left(\param^*\right) \leq \left(1-\delta_4\right)(\loss(\hat{\param})-\loss(\param^*))$. This completes the proof of the second part.
\end{proof}

\begin{proof}{\textbf{of Theorem~\ref{thm:convergence_relax}.}}
	Again, we use the notations in the proof of Theorem~\ref{thm:convergence_rate_relax}. Following the same derivation in~\eqref{eqn:objective_gap_hat_true}, we have: 
	$$
	\begin{aligned}
		\loss(\hat{\param})-\loss(\param^*)
		 \geq \frac{m}{2}\|\hat{\param} - \param^*\|^2 - \|\boldsymbol{d}^*_{\mathcal{I}_2^C}\|\|\hat{\param} - \param^*\|
		 = h(\|\hat{\param} - \param^*\|),
	\end{aligned}
	$$
	where $h(x) = \frac{m}{2}x^2 - \|\boldsymbol{d}^*_{\mathcal{I}_2^C}\|x$ is a quadratic function with respect to $x$. 

	Let $\delta_5 \coloneqq (m\vartheta)^{-1} \sqrt{s+s^*}\|\boldsymbol{d}^*\|$, when
	$\hat{\mathcal{A}}\supsetneq \mathcal{A}^{*}$, we have
        $$
		\|\hat{\param}-\param^*\| 
		\geq \|\param^*_{\mathcal{I}_1}\| 
		\geq \vartheta 
		= (m\delta_5)^{-1}\sqrt{s+s^*}\|\boldsymbol{d}^*\|_{\infty} 
		\geq (m\delta_5)^{-1}\|\boldsymbol{d}^*_{\mathcal{I}_2^C}\| 
		> m^{-1}\|\boldsymbol{d}^*_{\mathcal{I}_2^C}\|
        $$
	we have $\nabla h(\|\hat{\param} - \param^*\|) = m\|\hat{\param} - \param^*\| - \|\boldsymbol{d}^*_{\mathcal{I}_2^C}\| > 0$. Consequently, 
	\begin{equation}\label{eqn:objective-gap-hat-true-uncover}
        \begin{aligned}
		\loss(\hat{\param})-\loss(\param^*)
		\geq h(\vartheta) = \frac{m}{2}\vartheta^2 - \vartheta\|\boldsymbol{d}^*_{\mathcal{I}_2^C}\|
		\geq (\frac{1}{2}-\delta_5)m\vartheta^2 > 0            
        \end{aligned}
	\end{equation}
	On the other hand, by recursively applying the result of Theorem~\ref{thm:convergence_rate_relax}, we can derive that, when $t \geq \log_{(1-\delta_4)^{-1}}\left\lceil\frac{\mathop{\max}\{\loss(\param^0)-\loss(\param^*),0\}}{(\frac{1}{2}-\delta_5)m\vartheta^2}\right\rceil$, or equivalently  
	$$(1-\delta_4)^t\mathop{\max}\{\loss(\param^0)-\loss(\param^*),0\} < (\frac{1}{2}-\delta_5)m\vartheta^2,$$ 
	we have 
	$$\loss(\param^t)-\loss(\param^*) \leq (1-\delta_4)^t\mathop{\max}\{\loss(\param^0)-\loss(\param^*),0\}.$$
	Combining the two inequalities finally shows that $\loss(\hat{\param})-\loss(\param^*) < (\frac{1}{2}-\delta_5)m\vartheta^2$. Compared with~{ \eqref{eqn:objective-gap-hat-true-uncover}} when $\hat{\A} \subsetneq \A^*$, we can conclude that  
	$\supp(\param^t) \supseteq \mathcal{A}^{*}$ when $t \geq \log_{(1-\delta_4)^{-1}}\left\lceil\frac{\mathop{\max}\{\loss(\param^0)-\loss(\param^*),0\}}{(\frac{1}{2}-\delta_5)m\vartheta^2}\right\rceil$.
\end{proof}

\subsection{Proof of Lemmas~\ref{lemma1}-\ref{lemma2}.}

\begin{proof}{\textbf{of Lemma~\ref{lemma1}.}}
	For any $\mathcal{A}$ and $\boldsymbol{x}$ which satisfy $\vert \mathcal{A} \cup \operatorname{supp}(\boldsymbol{x})\vert  \leq k$, let $\mathcal{F}=\mathcal{A} \cup \operatorname{supp}(\boldsymbol{x})$ and consider the set $\Omega_S = \{a \in \mathbb{R}^{p} \mid \operatorname{supp}(a) \subseteq \mathcal{F} \}$. It is easy to see that $\Omega_S$ is a convex set.
	Because of $f$ is $m_k$-RSC and $M_k$-RSS, for $\forall a, b \in \Omega_S,\vert \operatorname{supp}(a - b)\vert  \leq\vert  \mathcal{F}\vert  \leq k$, we have:
	\begin{align*}
		\frac{m_{k}}{2}\|a-b\|^{2} \leq f(a)-f(b)-\nabla f(b)^{\top}[a-b] \leq \frac{M_{k}}{2}\|a-b\|^{2}.
	\end{align*}
	Hence, we have the following:
	\begin{equation}\label{lemma1:eq1}
		\begin{aligned}
			&f(a) \leq f(b)+\nabla f(b)^{\top}(a-b)+\frac{M_{k}}{2}\|a-b\|^{2}, \\
			&f(a) \geq f(b)+\nabla f(b)^{\top}(a-b)+\frac{m_{k}}{2}\|a-b\|^{2} .
		\end{aligned}
	\end{equation}
	
	We can find that { \eqref{lemma1:eq1}} are first order conditions for convexity of $\frac{M_{k}}{2}\|a\|^{2}-f(a)$ and $f(a)-\frac{m_{k}}{2}\|a\|^{2}$ at $\Omega_S$:
    \begin{align*}
        \begin{array}{ll}
            & [f(a)-\frac{m_{k}}{2}\|a\|^{2} ]- [f(b)-\frac{m_{k}}{2}\|b\|^{2} ] \geq \nabla [f(b)-\frac{m_{k}}{2}\|b\|^{2} ]^{\top}(a-b) \\
            \Leftrightarrow & f(a) \geq f(b)+\frac{m_{k}}{2} (\|a\|^{2}-\|b\|^{2} )  + [\nabla f(b)-m_{k} b ]^{\top}(a-b) \\
            \Leftrightarrow & f(a) \geq f(b)+\nabla f(b)^{\top}(a-b) \quad +\frac{m_{k}}{2} (\|a\|^{2}-\|b\|^{2}-2 b^{\top} a+2 b^{\top} b ) \\
            \Leftrightarrow & f(a) \geq f(b)+\nabla f(b)^{\top}(a-b)+\frac{m_{k}}{2}\|a-b\|^{2}.
        \end{array}
    \end{align*}

	Therefore, $f(a)-\frac{m_{k}}{2}\|a\|^{2}$ is convex at $\Omega_S$, so does $\frac{M_{k}}{2}\|a\|^{2}-f(a)$ similarly. Then, by second-order condition for convexity of the two functions, we have $\nabla^{2} [\frac{M_{k}}{2}\|a\|^{2}-f(a) ] \succeq 0 \text { and } \nabla^{2} [f(a)-\frac{m_{k}}{2}\|a\|^{2} ] \succeq 0$ for any $a \in \Omega_{\mathcal{F}}$. 
	In other words, we have $m_{k} I_{\vert  \mathcal{F}\vert } \preceq \nabla^{2} f(a) \preceq M_{k} I_{\vert  \mathcal{F}\vert }$ for $\forall a \in \Omega_S$. 
	Since $\nabla_{\mathcal{A}}^{2} f(\boldsymbol{x})$ is a principle submatrix of $\nabla_{\mathcal{F}}^{2} f(\boldsymbol{x})$, so we get the conclusion $m_{k} I_{\vert \mathcal{A}\vert } \preceq \nabla_{\mathcal{A}}^{2} f(\boldsymbol{x}) \preceq M_{k} I_{\vert \mathcal{A}\vert }$.
\end{proof}
\begin{proof}{\textbf{of Lemma~\ref{lemma2}.}}
	Denote $\mathcal{F}=\mathcal{A} \cup \mathcal{B}$, then 
	$$\nabla_{\mathcal{F}}^{2} f(\boldsymbol{x})=\left(\begin{array}{ll}\nabla_{\mathcal{A}}^{2} f(\boldsymbol{x}) & \nabla_{\mathcal{B}} \nabla_{\mathcal{A}} f(\boldsymbol{x}) \\ \nabla_{\mathcal{A}} \nabla_{\mathcal{B}} f(\boldsymbol{x}) & \nabla_{\mathcal{B}}^{2} f(\boldsymbol{x})\end{array}\right).$$
	Let $X=\nabla_{\mathcal{F}}^{2} f(\boldsymbol{x})-\frac{M_{k}+m_{k}}{2} I_{\vert \mathcal{F}\vert }$, by Lemma~\ref{lemma1}, we have $-\frac{M_{k}-m_{k}}{2} I_{\vert \mathcal{F}\vert } \preceq X \preceq \frac{M_{k}-m_{k}}{2} I_{\vert \mathcal{F}\vert } $.
	And we can get $X^{\top} X \preceq\left(\frac{M_{k}-m_{k}}{2}\right)^{2}I_{\vert \mathcal{F}\vert }$, then $\forall u \in \mathbb{R}^{\vert \mathcal{F}\vert },\left\|Xu\right\|=\sqrt{u^{\top} X^{T} X u} \leq \sqrt{\left(\frac{M_{k} - m_{k}}{2}\right)^{2}\|u\|^{2}}=\frac{M_{k}-m_{k}}{2}\|u\|$. 
	Due to $\mathcal{A} \cap \mathcal{B}=\emptyset$, $\nabla_{\mathcal{A}} \nabla_{\mathcal{B}} f(\boldsymbol{x})$ is submatrix of $X$, and hence $\forall v \in \mathbb{R}^{\vert \mathcal{A}\vert }$, let $\tilde{v}=$ $\left(0^{\top}, v^{\top}\right)^{\top} \in \mathbb{R}^{\vert \mathcal{F}\vert }$, we have:
	\begin{align*}
		\left\|\nabla_{\mathcal{A}} \nabla_{\mathcal{B}} f(\boldsymbol{x}) v\right\| \leq\|X \tilde{v}\| \leq \frac{M_{k}-m_{k}}{2}\|\tilde{v}\|=\frac{M_{k}-m_{k}}{2} \| v\|.
	\end{align*}
	For $\forall w \in \mathbb{R}^{\vert \mathcal{B}\vert }$, Cauchy-Schwartz inequality implies: 
	$$w^{\top} \nabla_{\mathcal{A}} \nabla_{\mathcal{B}} f(\boldsymbol{x}) v  \leq\|w\| \cdot\left\|\nabla_{A} \nabla_{B} f(\boldsymbol{x}) v\right\|  \leq \frac{M_{k}-m_{k}}{2}\|w\| \|v\|.$$
\end{proof}

\subsection{Proof of Proposition~\ref{prop:ising-rsc-rss}}
\begin{proof}{Proof of Proposition~\ref{prop:ising-rsc-rss}}
    Let
    $
    \phi_{k}(\mathbf{x}_{i}) \coloneqq \mathbb{P}(\mathbf{x}_{i k}\vert \mathbf{x}_{i 1}, \ldots, \mathbf{x}_{i k - 1}, \mathbf{x}_{i k + 1}, \ldots, \mathbf{x}_{ip}) = \frac{1}{1+\exp (-2\sum\limits_{l: l \neq k} \mathbf{x}_{i k}\mathbf{x}_{i l} \param_{k l})}, 
    $
    we can attain the differentiations of $\loss(\boldsymbol{\theta})$ by simple algebra:
    \begin{align*}
    	&(\nabla\loss)_{kl} 
    	\coloneqq 
    	\frac{\partial\loss(\boldsymbol{\theta})}{\partial\param_{kl}} 
            = -\frac{2}{n}\sum_{i=1}^n \mathbf{x}_{ik}\mathbf{x}_{il} (1-\phi_{k}(\mathbf{x}_i) + 1 -\phi_{l}(\mathbf{x}_i)),\\
    	&\left(\nabla^2 \loss\right)_{kl ; vw}
    	\coloneqq 
    	\frac{\partial^2 \loss(\boldsymbol{\theta})}{\partial \param_{kl} \partial \param_{vw}}
    	=
    	\begin{cases}
    		\frac{4}{n}\sum_{i=1}^n (1-\phi_k(\mathbf{x}_i))\phi_k(\mathbf{x}_i)+(1-\phi_l(\mathbf{x}_i))\phi_l(\mathbf{x}_i)
    		& \textup{if } k=v, l=w
    		\\ 
    		\frac{4}{n}\sum_{i=1}^n \mathbf{x}_{il} \mathbf{x}_{iw} (1-\phi_k(\mathbf{x}_i)) \phi_k(\mathbf{x}_i) 
    		& \textup{if } k=v, l \neq w
    		\\ 
    		0 
    		& \textup{other }
    	\end{cases}
    	.
    \end{align*}
	Let $\boldsymbol{u}$ be an arbitrary vector with dimension $p\times p$ satisfying $\vert \operatorname{supp}(\boldsymbol{u})\vert \leq 3s$ and $\| \boldsymbol{u} \|=1$. For any $k\in[p]$, denote $W_k \coloneqq \textup{diag}\{4(1-\phi_{k}(\mathbf{x}_1))\phi_{k}(\mathbf{x}_1), \dots,4(1-\phi_{k}(\mathbf{x}_n))\phi_{k}(\mathbf{x}_n)\}$ and $\textup{Tr}(W_k)$ as its trace. And define subvector of $\boldsymbol{u}$ that has length $p - 1$ as $\boldsymbol{u}_{-k} \coloneqq (\boldsymbol{u}_{k l})$ for $l \in [p] \text{ and } l \neq k$, then $\boldsymbol{u}_{-k}$ satisfies $\vert \operatorname{supp}(\boldsymbol{u}_{-k})\vert \leq \min\{3s, p-1\}$. Then, we have
	\begin{align*}
		\boldsymbol{u}^\top \left(\nabla^2 \loss\right)\boldsymbol{u} 
		=&\sum_{k=1}^p \sum_{\substack{v \neq k \\ w \neq k}} \boldsymbol{u}_{k v} \boldsymbol{u}_{k w} \frac{4}{n} \sum_{i=1}^n
		\left(
		\mathbf{x}_{i v} \mathbf{x}_{i w} (1-\phi_{k}(\mathbf{x}_i))\phi_{k}(\mathbf{x}_i) + \textup{I}(v=w)(1-\phi_{v}(\mathbf{x}_i))\phi_{v}(\mathbf{x}_i) 
		\right)
		\\
		=& \sum_{k=1}^p \sum_{\substack{v \neq k \\ w \neq k}} \boldsymbol{u}_{k v} \boldsymbol{u}_{k w} \left(\frac{1}{n}\mathbf{X}^\top W_k \mathbf{X} + \frac{1}{n}\textup{diag}\{\textup{Tr}(W_1), \dots, \textup{Tr}(W_p)\} \right)_{v w} 
		\\
		=& \sum_{k=1}^p \boldsymbol{u}_{-k}^\top  \left(\frac{1}{n}\mathbf{X}^\top W_k \mathbf{X} + \frac{1}{n}\textup{diag}\{\textup{Tr}(W_1), \dots, \textup{Tr}(W_p)\} \right)_{-k,-k} \boldsymbol{u}_{-k} 
		\\
		\leq& \sum_{k=1}^p \boldsymbol{u}_{-k}^\top  \left(\frac{1}{n}\mathbf{X}^\top \mathbf{X} + \textup{I}_p \right)_{-k,-k} \boldsymbol{u}_{-k} 
		\leq (1+M_{3s}) \sum_{k=1}^p \boldsymbol{u}_{-k}^\top  \boldsymbol{u}_{-k} 
		\leq (1+M_{3s}) \| \boldsymbol{u} \|,
	\end{align*}
	where the second inequality leverages Condition~\ref{con:ising:correlation}. Thus, $\loss(\cdot)$ is $(1+M_{3s})$-RSS.
	
	Note that $\Delta$ satisfies $\Delta \leq \phi_{k}(\mathbf{x}_i)(1-\phi_{k}(\mathbf{x}_i))$ for any $k \in [p]$. Then, with Conditions~\ref{con:ising:correlation} and~\ref{con:ising:parameter}, we can get $\boldsymbol{u}^\top \left(\nabla^2 \loss\right)\boldsymbol{u} \geqslant 8\Delta(1+m_{3s})$ by following the similar procedure for deriving its upper bound. So, $\loss(\cdot)$ is $8\Delta(1+m_{3s})$-RSC.
\end{proof}

\setcounter{theorem}{6}

\section{Parameter estimation analysis}
\subsection{Case I: with RIP-type condition}
\begin{theorem}\label{thm:L2_error_bound}
	Following Assumptions and notations in Theorem~\ref{thm:convergence}, there exists a constant $\delta_3>0$ such that the $\ell_2$-error of $\boldsymbol{\theta}^t$ is upper bounded by
	\begin{equation}\label{eq:l2_bound}
		\|\boldsymbol{\theta}^{t}-\boldsymbol{\theta}^{*}\|^2 \leq \frac{\delta_3}{m_{3s}} \delta_1^{-t} \vert \loss(\boldsymbol{\theta}^0)-\loss(\boldsymbol{\theta}^{*})\vert,
	\end{equation}
	if $\boldsymbol{\theta}^t$ is not the oracle solution.
\end{theorem}
\begin{proof}{Proof of Theorem~\ref{thm:L2_error_bound}}
	Let $\mathcal{A}^t = \operatorname{supp}\{\boldsymbol{\theta}^t\}$, $\mathcal{A}_{1}^t=\mathcal{A}^t \cap \mathcal{A}^{*}$, and $\mathcal{I}_1^t = (\mathcal{A}^t)^c \cap \mathcal{A}^{*}$,  
	then following the immediate results in Lemma~\ref{lemma:loss-now-next}, we have:
	\begin{align*}
		\begin{aligned}
			\|\boldsymbol{\theta}^t-\boldsymbol{\theta}^*\| 
			\stackrel{(\ref{eq:est_true_gap})}{\leq} m^{-1}\|\boldsymbol{d}_{\mathcal{A}^t}^*\|+(\frac{v}{m}+1)\|\param_{\mathcal{I}_1^t}^*\| 
            \leq (\frac{\sqrt{s} c}{m}+\frac{v}{m}+1)\|\param_{\mathcal{I}_1^t}^*\|
			\stackrel{(\ref{gapOfLoss:beforeSplicing})}{\leq} \frac{(\frac{\sqrt{s} c}{m}+\frac{v}{m}+1)\sqrt{ \vert  \loss(\boldsymbol{\theta}^t)-\loss (\boldsymbol{\theta}^{*} ) \vert }}{\sqrt{\frac{m}{2} - \sqrt{2s}c }},
		\end{aligned}
	\end{align*}
	where the second inequality results from Assumption~\ref{con:bound-gradient}. Furthermore,
	according to the result of Theorem~\ref{thm:convergence-rate}, we have
	\begin{align*}
		\|\boldsymbol{\theta}^t-\boldsymbol{\theta}^*\| 
		&\leq \frac{(\frac{\sqrt{s} c}{m}+\frac{v}{m}+1)\sqrt{\frac{\vert \loss(\boldsymbol{\theta}^0)-\loss(\boldsymbol{\theta}^{*})\vert }{m\delta_1^t}}}{\sqrt{\frac{1}{2}-\frac{\sqrt{2s}c}{m} }} 
		\leq \frac{(0.15x+1.0215)\sqrt{\frac{\vert \loss(\boldsymbol{\theta}^0)-\loss(\boldsymbol{\theta}^{*})\vert }{m\delta_1^t}}}{\sqrt{(\frac{1}{2}-\sqrt{2}(0.5215-0.35x))}} 
		\leq 2.31 \sqrt{\frac{\vert \loss(\boldsymbol{\theta}^0)-\loss(\boldsymbol{\theta}^{*})\vert }{m\delta_1^t}},
	\end{align*}
	where $x$ in the second inequality is defined as $x \coloneqq M / m$, and the last inequality follows from $x \in [1, 1.49]$ implied by Assumption~\ref{con:bound-gradient}.
\end{proof}

\subsection{Case II: without RIP-type condition}
\begin{theorem}\label{thm:L2_error_bound_relax}
	Following the same assumptions and notations in Theorem~\ref{thm:convergence_rate_relax}, $\ell_2$-error of $\param^t$ is upper bounded by
	\begin{equation}\label{eqn:L2_error_bound_relax}
        \begin{aligned}
		\|\param^t - \param^*\| &\leq \sqrt{\frac{2(1-\delta_4)^t}{m_{s+s^*}} \mathop{\max}\{\loss(\param^0)-\loss(\param^*),0\}} + \frac{2\sqrt{s+s^*}}{m_{s+s^*}}\|\nabla\loss(\param^*)\|_{\infty}.
        \end{aligned}
	\end{equation}
\end{theorem}
\begin{proof}{Proof of Theorem~\ref{thm:L2_error_bound_relax}.}
	Without the loss of generality, this proof adopts the notations used in proving Theorem~\ref{thm:convergence_rate_relax}. It is easily seen that
	\begin{align*}
		0 & \geq \loss(\param^*)-\loss(\hat{\param}) + (\hat{\param}-\param^*)^{\top}\boldsymbol{d}^* + \frac{m}{2} \|\hat{\param} - \param^*\|^2
		\geq \frac{m}{2}\|\hat{\param} - \param^*\|^2 - \|\boldsymbol{d}^*_{\mathcal{I}_2^C}\|\|\hat{\param} - \param^*\| + \loss(\param^*)-\loss(\hat{\param})\\
		& = \frac{m}{2}(\|\hat{\param} - \param^*\| - \frac{\|\boldsymbol{d}^*_{\mathcal{I}_2^C}\|}{m})^2 - \frac{\|\boldsymbol{d}^*_{\mathcal{I}_2^C}\|^2}{2m}+ \loss(\param^*)-\loss(\hat{\param}).
	\end{align*}
	Considering the first case that $\loss(\param^*)-\loss(\hat{\param}) \geq 0$, we can derive $\|\hat{\param} - \param^*\|\leq \frac{2}{m}\|\boldsymbol{d}^*_{\mathcal{I}_2^C}\|$. Moreover, due to $|\I_2^C| \leq s + s^*$, we have $\frac{2}{m}\|\boldsymbol{d}^*_{\mathcal{I}_2^C} \| \leq \frac{2\sqrt{s+s^*}}{m}\|\boldsymbol{d}^*\|_{\infty}$. Thus, it leads to  $\|\hat{\param} - \param^*\| \leq \frac{2\sqrt{s+s^*}}{m}\|\boldsymbol{d}^*\|_{\infty}$, which coincides with~{ \eqref{eqn:L2_error_bound_relax}}.

	On the other hand, when $\loss(\param^*)-\loss(\hat{\param}) < 0$, we can get
	\begin{align*}
		\|\hat{\param} - \param^*\| 
		\leq& \frac{1}{m}\|\boldsymbol{d}^*_{\mathcal{I}_2^C}\| + \frac{1}{m}\sqrt{\|\boldsymbol{d}^*_{\mathcal{I}_2^C}\|^2 - 2m(\loss(\param^*)-\loss(\hat{\param}))}
		\leq \frac{2}{m}\|\boldsymbol{d}^*_{\mathcal{I}_2^C}\| + \sqrt{\frac{2}{m}}\sqrt{\loss(\hat{\param})-\loss(\param^*)} \\
		\leq& \frac{2\sqrt{s+s^*}}{m}\|\boldsymbol{d}^*\|_{\infty} + \sqrt{\frac{2}{m}}\sqrt{(1-\delta_4)^t\mathop{\max}\{\loss(\param^0)-\loss(\param^*),0\}}
	\end{align*}
	where the last inequality follows from recursively applying Theorem~\ref{thm:convergence_rate_relax} and the fact that $f(\param^0) \geq f(\param^1) \geq \cdots \geq f(\param^{t-1})\geq f(\param^t) =: f(\hat{\param})$.
\end{proof}



\section{Selection of $k_{\max}$}\label{sec:max-splicing-size}
From Algorithm~\ref{algo:main}, $k_{\max}$ trade-offs the number of splicing iterations and the number of splicing operators with each splicing iteration. Specifically, a large $k_{\max}$ enlarges the number of splicing but reduces the number of splicing iterations; and vice versa.
Although the proof of Theorem~\ref{thm:recovery} requires $k_{\max}$ to be the true size of active set $s$, we can still get a similar result when $k_{\max}$ is smaller than $s$. We compare the results with different $k_{\max}$ to see this. 

The way of data generation is completely the same as the experiments in Section~\ref{sec:experiment-sota}. The results in Figure~\ref{Fig.k_max} show that the performance of our algorithm is insensitive to $k_{\max}$, and a medium-sized $k_{\max}$ makes the algorithm more computationally efficient. 
Thus, $k_{\max} = 5$ maybe a rule-of-thumb setting.


\begin{figure}[htbp]
	\centering 
	\includegraphics[width=1.0\textwidth]{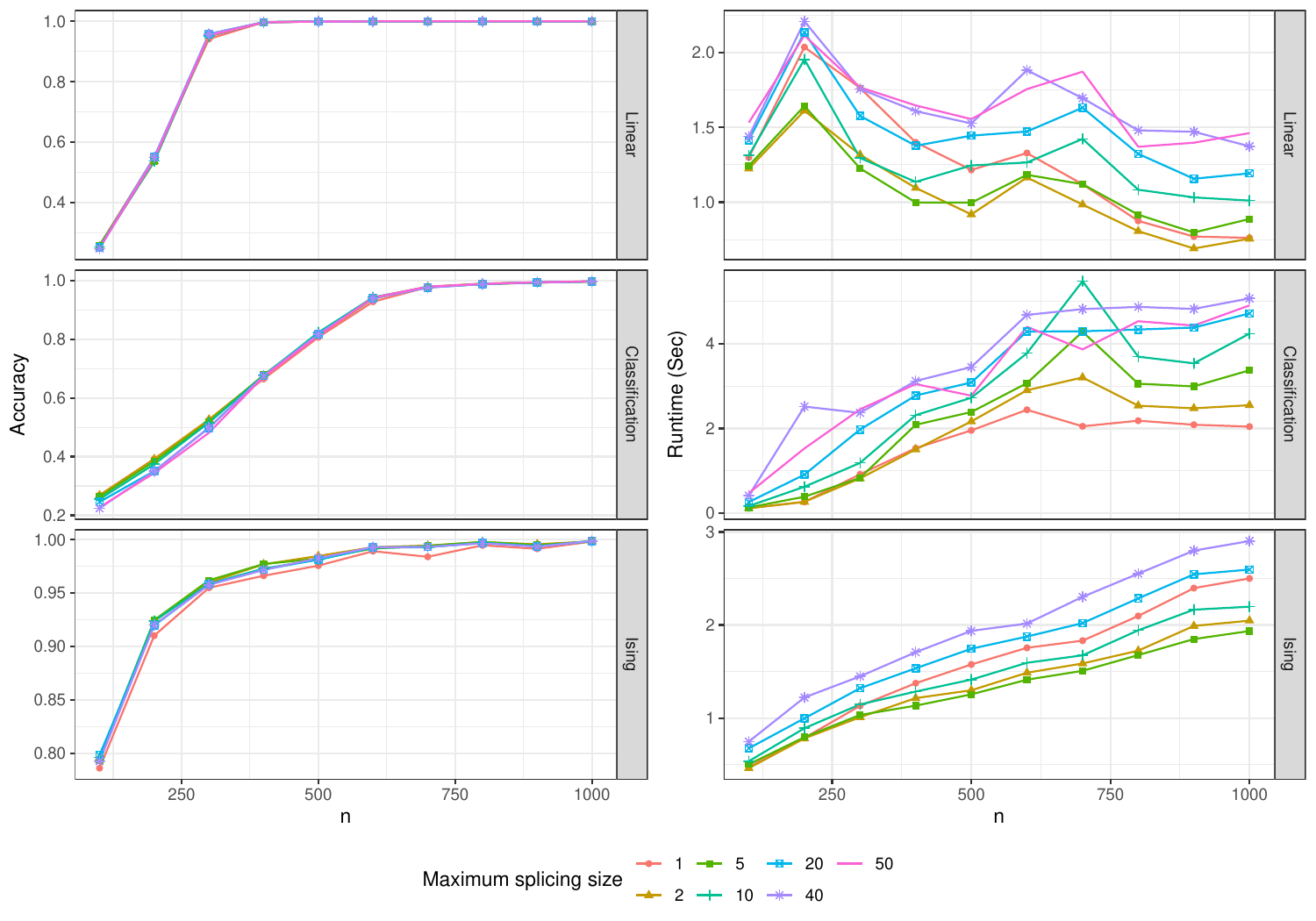}
    \vspace{-8pt}
	\caption{The mean of accuracy and runtime as sample size increases for different hyper-parameter $k_{\max}$. The experiment was independently repeated 20 times.} 
	\label{Fig.k_max} 
\end{figure}

\vskip 0.2in
\bibliography{reference}

\end{document}

%% file: diy_tex_command.tex
\newcommand{\loss}{f}
\newcommand{\supp}{\operatorname{supp}}
\newcommand{\A}{\mathcal{A}}
\newcommand{\I}{\mathcal{I}}
\newcommand{\param}{\boldsymbol{\theta}}
\newcommand{\zerovec}{\mathbf{0}}
\newcommand{\sacri}{\xi}
\newcommand{\exclu}{\mathcal{S}^\A}
\newcommand{\inclu}{\mathcal{S}^\I}
\newcommand{\Y}{\mathbf{Y}}
\newcommand{\y}{\mathbf{y}}
\newcommand{\X}{\mathbf{X}}
\newcommand{\x}{\mathbf{x}}
\newcommand{\eps}{\boldsymbol{\epsilon}}
\newcommand{\R}{\mathbb{R}}
\newcommand{\prob}{\mathbb{P}}
\newcommand{\skscope}{\texttt{scope} }
\newcommand{\cvxpy}{\texttt{cvxpy} }
\newcommand{\nablaf}{\nabla f}